%% file: master_flend_bias_mitigation.tex
\documentclass[10pt]{article}

\usepackage[usenames,dvipsnames,svgnames,table]{xcolor}

\usepackage{hyperref}

\usepackage[round,comma]{natbib}
\usepackage{enumitem}
\usepackage{sectsty,amsmath, amsfonts, amsthm, amssymb, mathtools,dsfont,mathrsfs}
\usepackage[lined,linesnumbered,ruled,commentsnumbered]{algorithm2e} 
\usepackage{fullpage}
\usepackage{subcaption}
\usepackage{float}	
\usepackage{bbm}
\usepackage{etoolbox}
\usepackage{graphics}
\usepackage[titletoc]{appendix}
\include{my_macros}

\numberwithin{equation}{section}

\title{Model-agnostic bias mitigation methods with regressor distribution control for Wasserstein-based fairness metrics}

\author{ Alexey Miroshnikov\thanks{Emerging Capabilities Research Group, Discover Financial Services, Riverwoods, IL} \textsuperscript{,$\!\!\!$}
\thanks{co-first author and corresponding author, alexeymiroshnikov@discover.com}
   \and
  Konstandinos Kotsiopoulos\textsuperscript{\specificthanks{1}}\thanks{co-first author, kostaskotsiopoulos@discover.com}  \and Ryan Franks\textsuperscript{\specificthanks{1},}\thanks{ryanfranks@discover.com} \and Arjun Ravi Kannan\textsuperscript{\specificthanks{1},}\thanks{arjunravikannan@discover.com}  }

\date{}

\begin{document}

\maketitle

\abstract{\input{abstract.tex}}

\smallskip




\input{introduction}

\input{preliminaries}

\input{model_selection}

\input{predictor_bias}

 \input{bias_mitigation}

\input{conclusion}

\input{acknowledgments}

\input{appendix}

\input{references}

\end{document}

%% file: my_macros.tex
\makeatletter
\def\namedlabel#1#2{\begingroup
    #2%
    \def\@currentlabel{#2}
    \phantomsection\label{#1}\endgroup
}
\newcommand{\specificthanks}[1]{\@fnsymbol{#1}}
\newcommand{\lrarrow}{\mathrel{\mathpalette\lrarrow@\relax}}
\newcommand{\lrarrow@}[2]{%
  \vcenter{\hbox{\ooalign{%
    $\m@th#1\mkern6mu\rightarrow$\cr
    \noalign{\vskip1pt}
    $\m@th#1\leftarrow\mkern6mu$\cr
  }}}%
}
\makeatother

\newcommand{\loss}{\mathcal{L}}
\newcommand{\family}{\mathcal{F}}
\newcommand{\eloss}{\bar{\mathcal{L}}}
\newcommand{\E}{\mathbb{E}}
\newcommand{\1}{\mathbbm{1}}
\renewcommand{\P}{\mathbb{P}}

\newcommand{\favdir}{\varsigma}
\def\del{\partial}

\newcommand{\Pcal}{\mathcal{P}}

\newcommand{\B}{\mathcal{B}}

\newcommand{\RR}{\mathbb{R}}
\newcommand{\PP}{\mathbb{P}}
\newcommand{\eps}{\varepsilon}
\newcommand{\Chi}{\mathcal{X}}

\newcommand{\Bias}{\text{\rm Bias}}

\newcommand{\fhat}{\hat{f}}

\newcommand{\Yhat}{\widehat{Y}}
\newcommand{\signbias}{ \widetilde{bias} }
\newcommand{\bias}{ bias }

\newcommand{\modbias}{ \Bias }

\newcommand{\SHAP}{ S\!H\!A\!P}

\newcommand{\ME}{ \text{\tiny \it ME}}
\newcommand{\CE}{ \text{\tiny \it CE}}

\newcommand{\IBE}{ \text{\tiny \it IBE}}
\newcommand{\vpdp}{v^{\ME}}
\newcommand{\vce}{v^{\CE}}

\newcommand{\A}{\mathcal{A}}

\def\clap#1{\hbox to 0pt{\hss#1\hss}}

\newtheorem{definition}{Definition}[section]
\newtheorem{lemma}{Lemma}[section]
\newtheorem{remark}{Remark}[section]



%% file: abstract.tex
This article is a companion paper to our earlier work \citet{Miroshnikov2020} on fairness interpretability, which introduces bias explanations. In the current work, we propose a bias mitigation methodology based upon the construction of post-processed models with fairer regressor distributions for Wasserstein-based fairness metrics. By identifying the list of predictors contributing the most to the bias, we reduce the dimensionality of the problem by mitigating the bias originating from those predictors. The post-processing methodology involves reshaping the predictor distributions by balancing the positive and negative bias explanations and allows for the regressor bias to decrease. We design an algorithm that uses Bayesian optimization to construct the bias-performance efficient frontier over the family of post-processed models, from which an optimal model is selected. Our novel methodology performs optimization in low-dimensional spaces and avoids expensive model retraining.


%% file: introduction.tex
\section{Introduction}\label{sec::introduction}




Machine learning (ML) techniques have become ubiquitous in their use across all industries, surpassing traditional statistical techniques in handling higher-dimensional data and generating models with increased predictive performance. Highly accurate ML models may potentially lack fairness, in the sense that model outputs may produce discriminatory outcomes among classes of protected attributes. Predictive models, and strategies that rely on such models, are subject to laws and regulations that ensure fairness. For instance, a hiring process in the United States (US) must comply with the Equal Employment Opportunity Act \citep{EEOA}. Similarly, financial institutions (FI) in the US that are in the business of extending credit to applicants are subject to the Equal Credit Opportunity Act \citep{ECOA}, the Fair Housing Act \citep{FHA1974}, and other fair lending laws.

These laws often specify protected attributes that FIs must consider when maintaining fairness in lending decisions. Examples of such attributes include race, gender, age, ethnicity, national origin, and marital status. Direct usage of protected attributes is prohibited under ECOA when training any ML model. However, other attributes can serve as ``proxies'' of protected attributes and thus learning from these may still cause the model to differentiate between subgroups which may potentially lead to discriminatory outcomes.

Bias, or unfairness, in the outputs of ML models can be measured by using a meaningful metric that assesses the difference between distributions of subpopulations in the output. \citet{Dwork2012} introduced this concept of bias at the level of a data distribution via randomized binary classifiers, which can be related to the Wasserstein distance via Kantorovich-Rubinstein duality arguments \citep{Miroshnikov2020}. 

This body of work is a companion paper to \citet{Miroshnikov2020} that introduced a fairness interpretability framework for measuring the model bias with respect to protected attributes at the level of the regressor distribution and explaining how each predictor contributes to that bias. The current article focuses on mitigating the model bias.

There is a comprehensive body of research on bias mitigation methodologies. \citet{Kamiran2009} introduced a classification scheme for learning unbiased models by modifying the biased data sets, without direct knowledge of the protected attribute. The works of \citet{Lahoti, Hashimoto} cite similar regulatory constraints on the collection of protected attribute data and also aim to achieve fairness without their use. The fairness metric chosen by the authors is the Rawlsian Max-Min Fairness principle \citep{Rawls}, which is used to formulate a minimax problem. The resulting objective function is then used in training to obtain a fairer model. Similar to \citet{Lahoti}, \citet{Zhang2018} formulated an adversarial learning framework to mitigate model bias. \citet{Jiang2020} introduced a procedure that involves reweighing data points and learning an unbiased model on the new dataset. Similarly, \citet{delBarrio} modify the data using optimal transport theory to then learn an unbiased model with the focus on two notions of fairness, disparate impact and balanced error rate. \citet{Feldman2015} proposed a scheme for removing disparate impact, in the sense of statistical parity, in classifiers by making data sets unbiased, and \cite{Hardt2015} introduced post-processing techniques removing discrimination in classifiers.


In \citet{Perrone2020} the bias methodology does not require  knowledge of the protected attribute $G$ either in training or prediction, and utilizes Bayesian optimization with fairness constraints on a wide range of models to learn ML hyperparameters that lead to fairer models. In \citet{Schmidt2021} the methodology randomly searches for ML hyperparameter configurations and builds the Pareto efficient frontier. \citet{Balashankar} propose the Pareto-efficient fairness, which identifies the point on the Pareto curve of subgroup performances closest to the fairness hyperplane, thus maximizing multiple group accuracy measures.

The bias mitigation techniques in many of the aforementioned articles are model-specific and often require knowledge of the protected attribute which is prohibited under regulatory constraints. This is not the case in the works of \citet{Perrone2020, Schmidt2021}. However, their methods can be rather computationally expensive since multiple retraining procedures are carried out. 

Our novel methodology uses Bayesian optimization in low-dimensional settings to construct the bias-performance efficient frontier over the family of post-processed models that do not explicitly depend on the protected attribute. Our approach completely avoids retraining, which allows for lower complexity. We reduce the problem dimensionality by making use of bias explanations that enable us to identify predictors contributing the most to the bias.

To outline our mitigation procedure, we introduce notation and import concepts described in the companion paper \citet{Miroshnikov2020}. Let $(X,G,Y)$, where $X\in\RR^n$ are predictors, $G \in \{0,1\}$ is the protected attribute, with the non-protected class $G=0$, and $Y$ is either a response variable with values in $\RR$ (not necessarily a continuous random variable) or binary one with values in $\{0,1\}$. We denote a trained model by $f(x)=\widehat{\E}[Y|X=x]$, assumed to be trained on $(X,Y)$ without access to $G$. We assume that there is a predetermined {\it favorable model direction}, denoted by $\uparrow$ and $\downarrow$; if the favorable direction is $\uparrow$ then the relationship $f(x)>f(y)$ favors the input $x$, and if it is $\downarrow$ the input $y$. In the case of binary $Y\in\{0,1\}$,  the favorable direction $\uparrow$ is equivalent to $Y=1$ being a favorable outcome, and $\downarrow$ to $Y=0$. To simplify the exposition, the main text focuses on the case of a binary protected attribute $G$. However, the framework and all the results in the article have a natural extension to the multi-labeled case.

In \citet{Miroshnikov2020},  we chose to assess regressor fairness using Wasserstein-based metrics. These metrics, which arise in optimal transport theory, assess regressor fairness at the level of the regressor values rather than proportions. In addition, the metric picks up changes in the geometry of the regressor distribution, unlike invariant metrics, such as Kolmogorov-Smirnov or Kullback-Leibler divergence, making it compatible with (non-invariant) ML performance metrics that utilize the regressor distribution (such as $L^2$, binomial deviance, or exponential loss), aiding in accurate bias-performance analysis together with distributional control.

In particular, the bias in the model output is measured as follows
\begin{equation*}
\begin{aligned}
\modbias_{W_1}(f|G)&=W_1(f(X)|G=0,f(X)|G=1)\\
& = \inf_{\pi \in \mathscr{P}(\RR^2)} \Big\{ \int_{\RR^2} |x_1-x_2| \, d\pi(x_1,x_2), \,\, \text{with marginals $P_{f(X)|G=0}, P_{f(X)|G=1}$}  \Big\}
\end{aligned}
\end{equation*}
where $W_1$ is the 1-Wasserstein metric, which measures the minimal cost of transporting one distribution into another; see \citet{Santambrogio2015}. The transport framework allows one to monitor the flow direction and measure the transport efforts in the favorable and non-favorable directions. In particular, the model bias can be decomposed as follows
\[
  \Bias_{W_1}(f|G) = \Bias_{W_1}^+(f|G) + \Bias_{W_1}^-(f|G),
\]
where the positive model bias $\Bias_{W_1}^+(f|G)$ measures the transport effort for moving points of the unprotected subpopulation distribution $f(X)|G = 0$ in the non-favorable direction and negative model bias $\Bias_{W_1}^-(f|G)$ in the favorable one. In \citet{Miroshnikov2020}, the above approach is generalized to a wide class of $W_1$-based metrics compatible with various fairness criteria.

The decomposition carries information that can help direct the bias mitigation. For example, in some FI applications, regulations might require to mitigate only the positive model bias (that quantifies the favorability of the model  with respect to the non-protected class). In general, the decomposition aids in reducing the complexity of the mitigation procedure.

In \cite{Miroshnikov2020},  utilizing the optimal transport approach, we introduced bias predictor attributions called bias explanations to understand how predictors contribute to the model bias. The bias explanation $\beta_i\geq 0$  for each predictor $X_i$  is decomposed into $\beta_i = \beta_i^+ +\beta_i^-$, with $\beta_i^{\pm}\geq 0$, where roughly speaking $\beta_i^+$ quantifies the predictor contribution to the increase of the positive model bias and decrease in the negative model bias cumulatively, and vice versa for $\beta_i^-$. The analysis of bias explanations allows one to separate predictors into three groups:

\begin{itemize}[label=$\bullet$]
\item $\beta_i^+ \gg \beta_i^-$.  Predictors that mainly contribute to transporting the non-protected subpopulation distribution in the favorable direction.
\item $\beta_i^+ \ll \beta_i^-$. Predictors that mainly contribute to transporting the non-protected subpopulation distribution in the non-favorable direction.
\item $\beta_i^+ \sim \beta_i^-$. Predictors that transport the non-protected class in the favorable or non-favorable direction depending on the sub-region of the predictor space.
\end{itemize}

In the current work, we further investigate bias explanations and their relation to predictor bias, which provides valuable insight in their use for bias mitigation. We present the main factors that affect the magnitude of the bias explanations, including the structure of the model, predictor's bias, and the shape of predictor distribution, and demonstrate that the predictor bias and its bias explanation are not equivalent. We also introduce a new type of marginal-based bias explanations, motivated by individual conditional expectations introduced in \cite{Goldstein et al}, which avoids many of the pitfalls of PDP-based bias explanations.

The major portion of this article is devoted to bias mitigation. At the core of our post-processing methodology is the relation between model bias and bias explanations. Specifically, we demonstrate that the model bias can be written as a linear superposition of certain game-theoretic bias explanations, which expresses the bias offsetting mechanism. This relation illustrates that the predictors are playing a game of tug-of-war, where some predictors contribute to pushing the non-protected subpopulation in a favorable direction and some others in a non-favorable one.

For bias mitigation, we construct a continuous-parameter family of post-processed models based upon perturbation of the trained regressor. The perturbation is performed by introducing predictor transformations for appropriate regressor inputs that allow us to rebalance bias explanations and reduce the model bias by taking advantage of bias offsetting. To reduce the dimensionality of the problem, we only transform the inputs corresponding to predictors that constitute the main drivers of the model bias. As a final step, we use Bayesian optimization in low-dimensional settings to reconstruct the bias-performance efficient frontier over the family of post-processed models.

We compare our mitigation procedure to that of \cite{Schmidt2021} and \cite{Perrone2020} on synthetic examples and show that our method performs better. It is our understanding that the reason for this is the fact that varying ML hyperparameters, without attacking directly the source of the bias in the joint distribution of predictors, lowers the model resolution and thus mainly impacts the performance while the effect on bias may be minimal.



\vspace{10pt}

\noindent{\bf Key steps of bias mitigation.}

\begin{itemize}
  \item[1.] Given $(W_1,\A)$-metric, we identify the most impactful predictors to the model bias by computing the positive and negative bias explanations. The list of most impactful predictors $M=\{i_1,i_2,\dots,i_m\}\subset N$ is generated, which can be further subdivided based on the relationship between positive and negative bias explanations. This procedure reduces the dimensionality of the problem.



\item[2.] Given a trained model $f$ and the list $M$ from step $1$, we construct an intermediate model of the form
  \[
    \tilde{f}(X;\alpha,x_M^*) = f(\bar{T}\circ (X_M), X_{-M}), \quad -M=N\setminus M,
  \]
 via the parameterized family of continuous transformations $\{\bar{T}(\cdot;a,x_M^*)\}$, with $a\in\RR^{km}$, $x_M^*\in\RR^m$, having the form
  \[
X_M \to \bar{T}(X_M; \alpha,x_M^*)=(T(X_{i_1}; \alpha^{1}, x_{i_1}^*),T(X_{i_2}; \alpha^{2},x_{i_2}^*),\dots,T(X_{i_m}; \alpha^{m},x_{i_m}^*)),
  \]
where the map $t \to T(t; a, t^* )$, $a\in \RR^k$, is strictly increasing and transforms the values by pulling them towards or pushing them away from the point $t^*$. We call the family of maps $\bar{T}(\cdot; \alpha, x^*_M)$ a compressive family. The above family allows us to rebalance the bias explanations by compressing (or expanding) each predictor in $X_M$. This family of maps is motivated by the scaling property $\Bias_{W_1}(cZ|G) = c\Bias_{W_1}(Z|G)$, $c \geq 0$.

  

\item[3.] 

Unlike the trained regressor, the post-processed one is no longer tied to data $(X,Y)$. In other words, while we can expect that $\fhat\approx \E[Y|X]$, the post-processed model may no longer approximate the true regressor. For this reason, we calibrate the model $\tilde{f}$ by constructing the post-processed model in the form $\bar{f} = C\circ \tilde{f}$ for an appropriate monotonic map $C(\cdot; \tilde{f},X,Y)$ that arises in the process of isotonic regression of either the trained model $\fhat$ or the response variable $Y$ onto the post-processed model $\tilde{f}$. The calibration procedure is a necessary step that ties the post-processed model to the data. For regressors of classification models, the procedure does not change invariant performance metrics such as AUC but only affects the regressor distribution. 


\item [4.] The final step is a model selection procedure, where an optimal model from the family of post-processed models $\mathcal{F} = \{\bar{f}(\cdot; \alpha, x_M^*), (\alpha,x_M^*) \in (A,\Chi_M) \subset \RR^{(k+1)m} \}$ is chosen given a specified bias-performance trade-off level. This is accomplished by solving the minimization problem with a fairness penalization term,
  \[
    \alpha_{*}(\omega) = \text{arg}\min_{\bar{f} \in \mathcal{F}} \Big\{ \E[\mathcal{L}(Y,\bar{f}(X)] + \omega \cdot \Bias_{W_1,\A}^{(w)}(\bar{f}|G,X) \Big\}, \quad \omega \ge 0.
  \]
Specifically, we design an algorithm that constructs the Pareto efficient frontier by solving the above problem via  Bayesian optimization; see \cite{Bergstra2011}. 



\end{itemize}

\paragraph{Structure of the paper.}  In Section 2, we introduce the requisite notation and fairness criteria for classifiers and the model bias definition as presented in \citet{Miroshnikov2020}. In Section 3, we describe the bias-performance trade-off and optimal selection mechanisms under fairness constraints for applications with regressor distribution control, including the discussion of compatibility of bias and performance metrics. In Section 4, we further investigate the bias explanations introduced in \citet{Miroshnikov2020} and discuss the connection between the bias explanations and the bias of predictors. Finally, in Section 5 we introduce in detail the four steps that form our bias mitigation methodology. In Appendix we discuss new types of bias explanations and provide auxiliary lemmas. 


%% file: preliminaries.tex
\section{Preliminaries}

\subsection{Notation and hypotheses}
We consider the joint distribution $(X,G,Y)$, where $X=(X_1,X_2,\dots,X_n) \in \RR^n$ are the predictors, $G\in \{0,1,\dots,K-1\}$ is the protected attribute and $Y$ is either a response variable with values in $\RR$ (not necessarily a continuous random variable) or a binary one with values in $\{0,1\}$. We encode the non-protected class as $G=0$ and assume that all random variables are defined on the common probability space $(\Omega,\mathcal{F},\PP)$, where $\Omega$ is a sample space, $\PP$ a probability measure, and $\mathcal{F}$ a $\sigma$-algebra of sets.

The true model and a trained one, which is assumed to be trained without access to $G$, are denoted by 
\[
f(X)=\E[Y|X] \quad \text{and} \quad \fhat(X)=\widehat{\E}[Y|X],
\]
respectively. We denote a classifier based on the trained model by 
\[
\Yhat_t=\Yhat_t(X;\fhat)=\1_{\{\hat{f}(X)>t\}}, \quad t \in \RR.
\]
In what follows, we suppress the symbol $\,\hat{}\,$ to denote the trained model, using it only when it is necessary to differentiate between the true model and the trained one. The same rule applies to classifiers.

Given a model $f$, the subpopulation cumulative distribution function (CDF) of  $f(X)|G=k$ is denoted by
\begin{equation}
F_k(t)=F_{f(X)|G=k}(t)=\PP(f(X)\leq t|G=k)
\end{equation}
and the corresponding generalized inverse (or quantile function) $F_k^{[-1]}$ is defined by:
\begin{equation}\label{subpopQF}
F_k^{[-1]}(p)=F_{f(X)|G=k}^{[-1]}(p)=\inf_{x \in \RR }\big\{ p \leq F_k(x) \big\}.
\end{equation} 

We assume that there is a predetermined {\it favorable model direction}, denoted by either $\uparrow$ or $\downarrow$. If the favorable direction is $\uparrow$ then the relationship $f(x)>f(z)$ favors the input $x$, and if it is $\downarrow$ the input $z$. The sign of the favorable direction of $f$ is denoted by $\varsigma_{f}$ and satisfies 
\[
\favdir_{f} = \left\{ 
\begin{aligned}
&1,& &\text{if the favorable direction of $f$ is $\uparrow$}&\\
-&1,& &\text{if the favorable direction of $f$ is $\downarrow$}\,.&
\end{aligned}
\right.
\] 
In the case of binary $Y$, the favorable direction $\uparrow$ is equivalent to $Y=1$ being a favorable outcome, and $\downarrow$ to $Y=0$.

In what follows we first develop the framework in the context of the binary protected attribute $G\in\{0,1\}$ and then extend it to the multi-labeled case. 

\subsection{Classifier and model biases }

\paragraph{Classifier bias.} When undesired biases concerning demographic groups (or protected attributes) are in the training data, well-trained models will reflect those biases. In what follows, we describe several definitions which help measure fairness of classifiers \citep{Hardt2015,Feldman2015,Miroshnikov2020}. 

\begin{definition}\label{def::parity} Suppose that $Y$ is binary with values in $\{0,1\}$ and $Y=1$ is the favorable outcome. Let $\Yhat$ be a classifier.
\begin{itemize}[label=$\bullet$]

\item  $\Yhat$ satisfies statistical parity if $\PP(\Yhat=1|G=0) = \PP(\Yhat=1|G=1).$

\item $\Yhat$ satisfies equalized odds if $\PP(\Yhat=1|Y=y,G=0) = \PP(\Yhat=1|Y=y,G=1)$, $ y\in\{0,1\}$

\item $\Yhat$ satisfies equal opportunity if $\PP(\Yhat=1|Y=1,G=0) = \PP(\Yhat=1|Y=1,G=1)$
\item Let $\mathcal{A}=\{A_j\}_{j=1}^M$ be a collection of disjoint subsets of $\Omega$. $\Yhat$ satisfies $\mathcal{A}$-based parity if
\begin{equation}\label{genparity}
\PP(\Yhat=1|A_m, G=0)=\PP(\Yhat=1| A_{m}, G=1 ), \quad m \in \{1,\dots,M\}.
\end{equation}
\end{itemize}
\end{definition}

The statistical parity requires that the proportions of people   in the favorable class $\Yhat=1$ within each group $G=k,k\in\{0,1\}$ are the same. The  equalized odds constraint requires the classifier to have the same misclassification error rates for each class of the protected attribute $G$ and the label $Y$.  Equal opportunity constraint requires the misclassification rates to be the same for each class $G=k$ only for the individuals labeled as $Y=1$. The $\mathcal{A}$-based parity requires subpopulations to have the same proportions within each event $A_m \in \mathcal{A}$; it is a generalization of the first three criteria. For instance, taking $\mathcal{A}=\{\Omega\}$ gives statistical parity, and $\mathcal{A}=\{\{Y=0\},\{Y=1\}\}$ gives equalized odds.

The classifier bias can be defined as a deviation from the statistical parity. In particular, for statistical parity we have the following definition.
\begin{definition}[\bf classifier bias]\label{def::statbias} Let $f$ be a model, $X\in\RR^n$ predictors, $G \in \{0,1\}$ protected attribute, $G=0$ non-protected class, and  $\favdir_f$ the sign of the favorable direction of $f$. 
\begin{itemize}[label=$\bullet$] 
\item The signed classifier (or statistical parity) bias for a threshold $t \in \RR$ is defined by
\[
\begin{aligned}
& \signbias^{C}_t(f|X,G) \\ & = \big( \PP(Y_t=\1_{\{\favdir_f=1\}}|G=0)-\PP(Y_t=\1_{\{\favdir_f=1\}}|G=1) \big) \cdot \favdir_f 
=( F_1(t)-F_0(t) ) \cdot \favdir_f.
\end{aligned}
\]

\item The classifier bias is defined by $\bias^C_t(f|X,G)=|\signbias^C_t(f|X,G)|.$
\end{itemize}
\end{definition}

We say that $Y_t$ favors the non-protected class $G=0$ if the signed classifier bias is positive. Respectively, $Y_t$ favors the protected class $G=1$ if the signed classifier bias is negative.

To take into account the geometry of the model distribution we define the quantile bias.
\begin{definition}[\bf quantile bias]\label{def::quantbias}
Let $f,X,G,\favdir_f$ and $F_k$ be as in Definition \ref{def::statbias}. Let $p \in (0,1)$.
\begin{itemize}[label=$\bullet$]
\item The signed $p$-th quantile is defined by $\signbias^{Q}_p(f|X,G) = \big( F^{[-1]}_0(p)-F^{[-1]}_1(p) \big) \cdot \favdir_f
$
	\item The $p$-th quantile bias is defined by  $\bias^Q_p(f|X,G)=|\signbias^Q_p(f|X,G)|.$
\end{itemize}
\end{definition}

The two definitions above can be easily generalized to any type of parity introduced in Definition \ref{def::parity}; see \citep{Miroshnikov2020}. For this reason we will primarily work with the statistical parity criterion.

\paragraph{Model bias.} Following the ideas in the companion paper \citet{Miroshnikov2020}, we define the model bias as a cost of transporting one subpopulation of the model into another. This gives the following.

\begin{definition}[\bf model bias]\label{def::modbias} Let $f$, $X$, $G$, $\favdir_f$ be as in Definition \ref{def::statbias} and $\mathcal{A}$ as in Definition \ref{def::parity}.

\begin{itemize}[label=$\bullet$] 

\item Given positive weights $w=\{w_m\}_{m=1}^M$, the $(W_1,\mathcal{A},w)$-based model bias is defined by
\begin{equation*}\label{genparmodbias}
	\Bias_{W_1,\mathcal{A}}^{(w)}(f|X,G) = \sum_{m=1}^{M} w_{m} W_1 \big(f(X)|\{A_m, G=0\},f(X)|\{A_{m},G=1\}\big), \quad  w_{m} > 0,
\end{equation*}
where $W_1$ stands for the Wasserstein distance and the weights satisfy $\sum_{m=1}^{M} w_{m}=1$.

\item In the special case when $\mathcal{A}=\{\Omega\}$, we define $W_1$-based model bias by
\begin{equation}\label{modbias}
\Bias_{W_1}(f|X,G) = W_1\big(f(X)|G=0,f(X)|G=1\big).
\end{equation}
\end{itemize}
\end{definition}

By assumption $f(X) \in \RR$, which allows one to express $W_1$-based bias as an integrated classifier bias. Specifically, we have the following connection of the model bias and the classifier and quantile biases:
\[
\Bias_{W_1}(f|X,G)=\int_{[0,1]} \bias^Q_p(f|X,G) dp=\int_{\RR} \bias^C_t(f|X,G) dt.
\]

The total transport cost can be decomposed into the transport cost of moving the points of the non-protected subpopulation distribution $f(X)|G=0$ in the non-favorable direction and in the favorable one, respectively. This gives rise to the decomposition of the bias into positive and negative components:
\[
\Bias_{W_1}(f|G)=\Bias_{W_1}^+(f|G)+\Bias_{W_1}^-(f|G), \quad \Bias_{W_1}^{\pm}(f|G)= \int_{\mathcal{P}_{\pm}} \bias^{Q}_p(f|G) \, dp
\]
where $\Pcal_{\pm} = \{p \in [0,1]: \pm\signbias^{Q}_p(f|X,G)=\pm(F^{[-1]}_{f(X)|G=0}-F^{[-1]}_{f(X)|G=1})\cdot \favdir_{f} > 0 \}$. Measuring the two flows allows for taking into account the sign of the bias across quantile subpopulations, which provides us with a more informative perspective on its origin, and aids in bias mitigation.

\begin{definition}[\bf fair model]\label{fairmodel}
Let $X,G,f,\favdir_f$ be as in Definition \ref{def::statbias}. We say that the model $f$ is fair in the $(W_1,\A,w)$-metric if $\Bias_{W_1,\A}^{(w)}(f|X,G)=0$. We say that $f(X,G)$ is fair up to $\epsilon$ $(W_1,\A,w)$-bias  if $\Bias_{W_1,\A}^{(w)}(f|X,G)\leq \epsilon$. 
\end{definition}

\subsection{Bias explanations}

\subsubsection{Explainers.} To evaluate bias attributions we rely on predictor explanation techniques. A generic single feature explainer of $f$ that quantifies the attribution of each predictor $X_i$ to the model value $f(X)$ is denoted by 
$E(X; f) = \{E_i(X,f)\}_{i=1}^n.$

A straightforward way of setting up an explainer $E_i$ is by specifying each component via a conditional or marginal expectation $E_i(X;f)=v(\{i\};X,f)$, $v \in \{\vce,\vpdp\}$, where
\begin{equation}\label{shapgame}
	\vce(S; X, f)=\E[f|X_S], \quad \vpdp(S;X,f)=\E[ f(x_S,X_{-S}) ]|_{x_S=X_S}.
\end{equation}
These simple explainers however do not handle interactions well and may lead to inconsistent attributions \citep{Goldstein et al}. For that reason, one may choose to work with game theoretical explainers.

One such explainer is based on the Shapley value 
\begin{equation}\label{shapform}
\varphi_i[v]= \sum_{S \subseteq N \backslash \{i\}} \frac{s!(n-s-1)!}{n!} [ v(S \cup \{i\}) - v(S) ], \quad  s=|S|, \,  n=|N|,
\end{equation}
by utilizing the marginal or conditional game $v \in \{\vce,\vpdp\}$. 

In the presence of dependencies in predictors, the two games differ. The conditional game explores the data by taking into account dependencies, while the marginal game explores the model $f$ in the space of its inputs, ignoring the dependencies. In particular, it can be shown that the explanations $\varphi_i[\vce]$ are consistent with the data, that is, the map $\omega \to f(X(\omega))$, while $\varphi_i[\vpdp]$ are consistent with the structure of the model, that is, the map $x \to f(x)$ \citep{Chen-Lundberg,Kotsiopoulos2020}. 

The choice between the two games is application specific. In scientific applications it is crucial to understand the true reasons behind observed data regardless of the model structure and hence $\vce$ might be preferable. In other applications, where the model structure is required to be explained, the game $\vpdp$ should be used.

In this work we are designing a bias mitigation methodology based on a post-processing technique that alters the model $f$. For this reason the marginal game is more appropriate as we will see later.

\subsubsection{Basic bias explanations.}

 Following \citet{Miroshnikov2020}, we compute bias attributions as follows. Given a single feature explainer $E(X,f)$, the bias attribution of the predictor $X_i$ is defined as the minimal cost of transporting $E_i(X,f)|G=0$ to $E_i(X,f)|G=1$:
\begin{equation}\label{biasexpl}
  \beta_i(f|X,G; E_i) = W_1(E_i(X;f)|G=0 , E_i(X;f)|G=1 ) = \int_0^1 |F_{E_i|G=0}^{[-1]}-F_{E_i|G=1}^{[-1]}| \, dp.
\end{equation}

The transport theory characterization allows one to take into account the sign of the bias when measuring the bias impact of a given predictor. In particular, the positive and negative bias explanations are defined as the transport cost of $E_i|G=0$ in non-favorable and favorable directions, respectively:
\[
    \beta_i^{\pm}(f|X,G; E_i) = \int_{\Pcal_{i\pm}} (F_{E_i|G=0}^{[-1]}-F_{E_i|G=1}^{[-1]}) \cdot \favdir_{f} \, dp,
\]
where $\Pcal_{i\pm} = \{p \in [0,1]: \pm\signbias^{Q}_p(E_i|X,G)=\pm(F^{[-1]}_{E_i|G=0}-F^{[-1]}_{E_i|G=1})\cdot \favdir_{f} > 0 \}$.


The bias explanation $\beta_i\geq 0$  of $X_i$  is then decomposed into  $\beta_i = \beta_i^+ +\beta_i^-$, with $\beta_i^{\pm} \geq 0$, where roughly speaking $\beta_i^+$ quantifies the predictor contribution to the increase of the positive model bias and decrease in the negative model bias cumulatively, and vice versa for $\beta_i^-$; for more details see Section \ref{sec::modsuperpos}.

In particular, if $\beta_i^{+}>0$ and $\beta_i^{-}=0$,  we say that $E_i(X)$ strictly favors the non-protected class $G=0$, while if $\beta_i^{-}>0$ and $\beta_i^{+}=0$, $E_i(X)$ strictly favors the protected class $G=1$. If $\beta_i^{\pm}>0$ we say that $X_i$ has mixed bias explanations. The net bias explanation of the predictor $X_i$ is defined as the difference $\beta_i^{net}=\beta^{+}-\beta_i^{-}$.

\subsubsection{Shapley bias explanations.} \citet{Strumbelji2014} introduced a game-theoretic approach for computing predictor attributions to the model value. In this approach predictors are treated as players and an appropriately designed game is based on the trained model and its modification. Motivated by the aforementioned article, we design an appropriate bias game. Specifically, we treat predictors as players that push/pull sub-populations distributions apart. This gives rise to a Wasserstein-based bias game that evaluates the effort of predictors when they join various coalitions:
\[
 v^{bias}(S; G, E(\cdot\,; X,f))=W_1(E(S; X,f)|G=0,E(S; X,f)|G=1), \quad S\subset\{1,2,\dots,n\},
 \]
where $E(S;X,f)$ is a group explainer that quantifies the attribution of each predictor $X_S$ to the model value. The Shapley-bias explanations are then defined as the Shapley value $\varphi[v^{bias}]$ of the game $v^{bias}$. 

Measuring the cost of transporting group explanations for the unprotected class in non-favorable and favorable directions leads to positive and negative bias games $v^{bias+}$ and $v^{bias-}$, which, in turn, yields positive and negative Shapley-bias explanations $\varphi[v^{bias+}]$ and $\varphi[v^{bias-}]$.

One typical approach to choose group explainers is to set $E_S(\cdot;X,f) = v \in\{\vce, \vpdp\}$, or alternatively to employ trivial group explainers $E_S(\cdot;X,f)$ based on sums of the Shapley values:
\begin{equation}\label{grouppdpshap}
\varphi_S[v]=\varphi_S(X;f,v)=\sum_{i \in S} \varphi_i(X;f,v) \quad \text{where} \quad v \in \{ \vce, \vpdp\}.
\end{equation}

%% file: model_selection.tex
\section{Model selection under fairness constraints}\label{sec::modselect}

In this section we discuss an optimal model selection mechanism via Pareto efficient frontier and compatibility of bias-performance metrics.  To this end, consider a collection of parametrized models.

\begin{equation}\label{model_family}
\mathcal{F}=\Big\{f_{\theta}(x)=\widehat{\E}[Y|X=x;\theta], \, \theta \in \Theta \subset \RR^m \Big\}
\end{equation}
where $\Theta$ denotes a parameter space. In the case of binary $Y \in \{0,1\}$, the regressor $f_{\theta}$ becomes the classification score $\widehat{\P}(Y=1|X;\theta)$ with classifiers defined through the regressor.

Ideally, we would like to identify models in $\mathcal{F}$ whose  bias-performance trade-off is optimal, that is, among models with similar performance, we would like to identify those that are the least biased and vice versa.  To this end, we let $\loss(y,f(x))$ denote a loss function and define the (regressor) performance to be  $-\eloss(Y,f(X))$, and let $\Bias_{D}(f|X,G)$ be the metric that measures the model bias by computing the distance $D(\cdot,\cdot)$ between subpopulation distributions of the regressor.

For classification models, in addition to an assessment of regressor performance, it is common to assess the models based on the performance of classifiers $\{Y_t(X;f)\}_{t \in \RR}$ they induce using (classification) model performance metrics such as AUC.

The choice of metrics for both bias and performance depends highly on the type of decisions being made as well as on the strategies that rely on the output of predictive models. Given a model $f$, we consider three types of strategies associated with it:
\begin{itemize}

  \item[(S1)] Strategies that rely on the model regressor, and hence utilize  the model distribution; this is relevant to both regressor and classification models.

  \item[(S2)] Strategies that rely on classifiers $\{Y_t(X;f), t \in \RR \}$ obtained by thresholding a regressor, without explicit use of the regressor distribution (or probabilities).

  \item[(S3)] Strategies that rely on classifiers $\{Y_t(X;f), t \in \RR \}$ induced by the regressor, with explicit use of the regressor distribution (or probabilities).

\end{itemize}

Strategies (S1)-(S3) determine the choice of appropriate metrics.  In our work, we primarily focus on strategies of type (S1) and study regressor bias-performance trade-off with distributional control. In what follows, we discuss the issue of compatibility among bias-performance metrics so that a proper trade-off can be realized.

\begin{figure}
 \centering
  \begin{subfigure}[t]{0.6\textwidth}
    \centering
    \includegraphics[width=\textwidth]{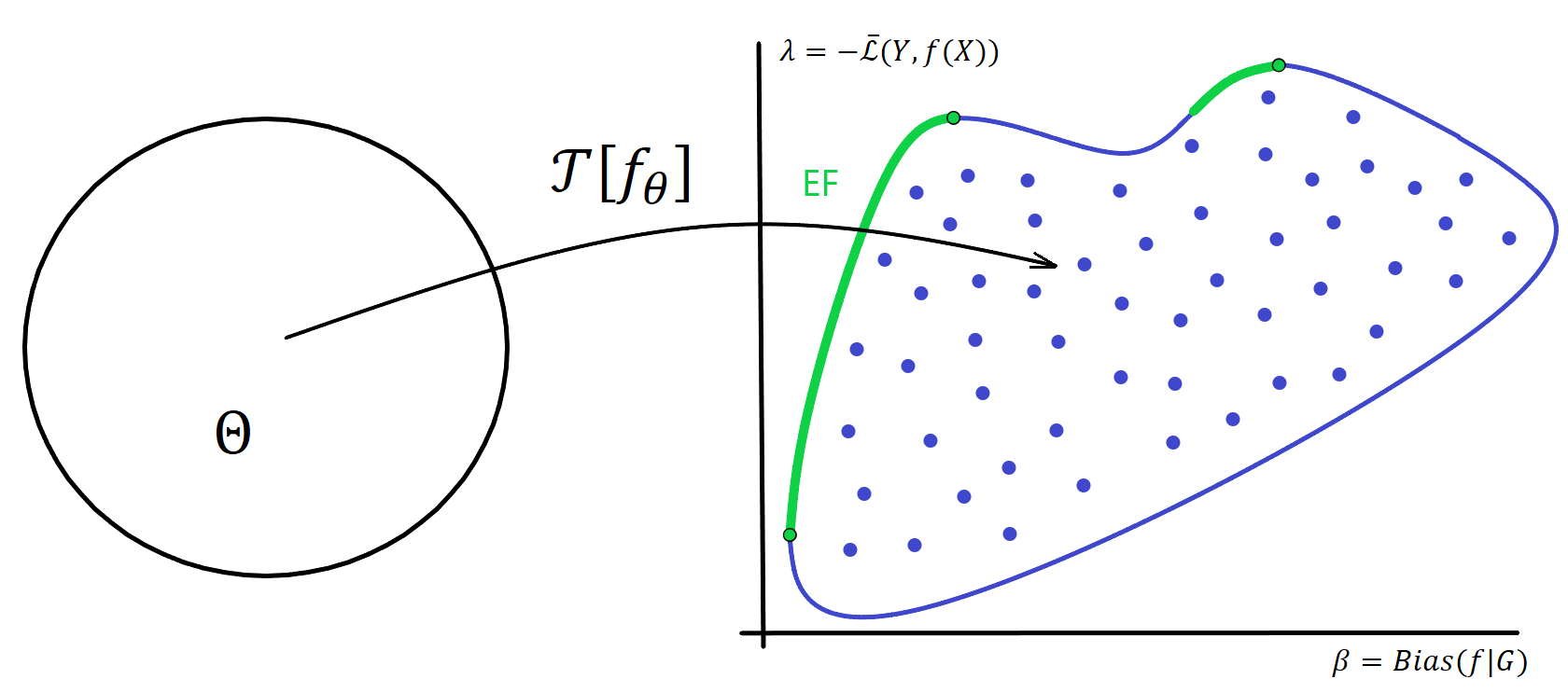}
  \end{subfigure} 
  \caption{ BP-efficient frontier.}\label{fig::efffrontmock}
\end{figure}

\subsection{Selection based on regressor bias-performance tradeoff }\label{sec::rbp_tradeoff}


Selection for decisions of type (S1) involves the control of the regressor bias and regressor performance. Note that any loss function $\loss(y,f(x))$ evaluates the distance between the predicted values and the data points and as a consequence performance metric given by $\eloss(Y,f(X))$ will respond to the change in geometry of the model distribution $f_{\theta}(X)$ caused by the change in $\theta$. This implies that $\eloss(Y,f(X))$ is non-invariant with respect to monotone transformations $T$ of the model $f$, that is, in general $\eloss(Y,f(X)) \neq \eloss(Y,T(f(X)))$.

Thus, for compatibility, an appropriate model bias metric should also respond to the change in geometry of the distribution. In particular, one would expect that when sub-population distributions move closer to each other, the performance and the bias should decrease and vice versa. The minimum requirement for the bias metric compatibility with the regressor performance is to require the bias metric to be non-invariant with respect to monotone transformations. This property will ensure that a proper regressor bias-performance trade-off is realized; this is especially important in the world of post-hoc corrective regressors, where one has the ability to alter the model.

In particular, Wasserstein-based model bias $\Bias_{W_1}(f|G)$ is a suitable candidate satisfying the above requirement. In contrast, the model bias based on Kolmogorov-Smirnov (KS), or other invariant statistical distance such as Kullback-Leibler (KL) divergence, does not satisfy this requirement: any scaling of the model results in punishing the loss function but does not affect the change in the model bias: $\Bias_{KS}(T(f)|G)=\Bias_{KS}(f|G)$. Thus, the bias-performance trade-off might not be properly realized when pairing KS metric with regressor performance given by the expected loss $\eloss$ because such bias metric is insensitive to monotone changes in the regressor.

To illustrate why this is the case, let $f_*$ be a model and consider the family of models
\[ 
\mathcal{F}=\{f: f=T(f_*), T: \RR \to \RR,  \text{$T$ is smooth, strictly increasing}\}.
\] 
Observe that $\Bias_{KS}(f_*|X,G)=\Bias_{KS}(\family|X,G)$, meanwhile $\eloss(Y,T\circ f_*(X))$ can be made arbitrarily large as long as the loss function is unbounded with respect to the second argument (as it tends to infinity). Thus, this bias-performance pairing is non-elastic, which is undesirable.  Hence, it is crucial to pair bias-performance metrics that showcase elasticity with respect to each other in order for a proper trade-off to be realized. Note that similar compatibility arguments can be made for compatilibity of metrics used in (S2) and (S3).

\begin{definition} 
We say that the pair of (bias,performance) metrics are compatible if both of them are either non-invariant with respect to the monotone transformation or both invariant.
\end{definition}

In what follows, we assume that the two metrics used for regressor bias-performance analysis are compatible and primarily work with $\Bias_{W_1}(f|G)$ or its generalized version that includes an appropriate link function.

Provided with an appropriate model bias metric and a regressor performance metric pair, we design a selection mechanism that relies on the concept of  Pareto optimality \citep{Luc2016}, which is related to the idea of the efficient frontier introduced by \citet{Markowitz} in the context of financial portfolio selection. The Pareto efficient frontier was later adapted in the context of ML bias mitigation; see \citet{Schmidt2019, Schmidt2021, Balashankar,Perrone2020}.

  \begin{definition}[\bf efficient frontier]\label{def::rbpmap}
 Let $\B(f|X,G)$ be the metric that measures the model bias and $\loss(y,f(x))$ be the loss function. Let $\mathcal{F}$ be as in \eqref{model_family} and
  \[
  U[\mathcal{F}; \B, \loss ]=\big\{(\beta,\lambda): \, \beta=\Bias(f_{\theta}|X,G), \lambda= -\eloss(Y,f_{\theta}(X)), \theta \in \Theta \big\}.
  \]
Suppose $\Theta \subset \RR^m$ is connected and the map $
\theta \to \big(\B(f_{\theta}|G),-\eloss(Y,f_{\theta}(X))\big)$ is continuous on $\Theta$. Then, the bias-performance efficient frontier (BPEF) is defined by:
\[
{\rm EF}[\family; \B, \loss] = \Big\{ (\bar{\beta},\bar{\lambda}) \in \del U:  (\beta,\bar{\lambda})\in \bar{U} \Rightarrow  \beta\geq \bar{\beta}, \, (\bar{\beta},\lambda)\in \bar{U} \Rightarrow  \lambda \leq \bar{\lambda}  \Big\}.
\]

\end{definition}

Note that the efficient frontier (under the assumptions of Definition \ref{def::rbpmap}) is always well-defined and non-empty if the boundary of $U[\family]$ is nonempty; see Figure \ref{fig::efffrontmock}. Furthermore, the set $U[\family]$ in general is not closed and hence it may happen that the efficient frontier itself may contain no points from $U[\family]$, but one can find models in $\family$ whose image under $\mathcal{T}$ is arbitrarily close to the efficient frontier. In the special case when the family  $\family$ is discrete one can construct an efficient frontier as the frontier of the convex hull of $U[\family]$ or that of the $\alpha$-shape.


To reconstruct the bias-performance efficient frontier, one can consider an optimization problem where the loss function includes the model bias as a penalization term:
\begin{equation}\label{efrontoptim}
\begin{aligned}
\theta_*(\omega) &= \underset{\theta \in \Theta}{\rm argmin} \Big\{ \eloss(Y, f_{\theta}(X)) + \omega \B(f_{\theta}|G,X) \Big\}, \quad \omega\geq 0.
\end{aligned}
\end{equation}

The above problem is not trivial for the following reasons. First, the optimization is in general non-convex; this is a direct consequence of the loss and bias terms in the objective function. Second, the dimension of the predictors $X$ and parameter $\theta$ can be large, increasing the complexity of the problem. Finally, in applications where the map $\theta \to f_{\theta} \in \family$ is non-smooth,  utilizing gradient-based optimization techniques might not be feasible. 

There are several approaches for \eqref{efrontoptim} that have been proposed in the literature. One approach is to incorporate the fairness constraint directly into ML training \citep{Dwork2012, Feldman2015, Zemel2013, Woodworth2017}. The second approach leaves the ML training procedure untouched; instead, it utilizes the hyperparameter search, which can be done either by randomly selecting hyperparameters, or by utilizing Bayesian hyperparameter search, or by employing feature engineering \citep{Bergstra2011, Schmidt2019, Perrone2020}. By design this method is both model-agnostic and metric-agnostic, although it has several limitations. For each $\theta \in \RR^m$, the loss function is optimized without consideration of fairness, which may lead to a very narrow (along fairness axis) efficient frontier; see Section 5. Furthermore, if the search space $\Theta$ is very large and the dataset is high-dimensional with  large number of observations, retraining can be computationally expensive.

%% file: predictor_bias.tex
\section{Relation between bias explanations and predictor bias}\label{subsec::bias_expl_pred}


\subsection{Effect of predictor bias on the model bias}

In this section, we investigate the relationship between bias explanations and the predictor bias, which can provide valuable insight on the use in bias mitigation. To this end we define the bias in predictors as follows.

\begin{definition}
  Let $G$ be a protected attribute and $\A$ be as in Definition \ref{def::modbias}.  Let $D(\cdot,\cdot)$ be a metric on the space of probability measures $\mathscr{P}_q(\RR)$, with $q \geq 0$. Suppose that $\E[|Z|^q]$ is finite.
\begin{itemize}
  \item [$(i)$] The bias of the random vector $Z\in \RR^k$ in the $(D,\A,w)$-metric is defined by
  \begin{equation*}
    \Bias_{D,\A}^{(w)}(Z|G) = \sum_{m=1}^M w_m D(P_{Z|\{A_m,G=0\}},P_{Z|\{A_m,G=1\}}).
  \end{equation*}

\item [$(ii)$] We say that $Z$ is unbiased in the $(D,\A,w)$-metric if $\Bias_{D,\A}^{(w)}(Z|G)=0$.
\end{itemize}
\end{definition}

\begin{lemma}
  $Z$ is unbiased in the $(D,\A,w)$-metric if and only if $Z|A_m$ is independent of $G$ for each $m$. As a consequence, if $\Bias_{D_1,\A}(X|G)=0$, then $\Bias_{D_2,\A}(f|X,G)=0$ for any two metrics $D_1$ and $D_2$.
\end{lemma}
\begin{proof}
  The proof follows from Lemma \ref{lmm::indepwasserstconn} provided in the appendix.
\end{proof}

The main message of the analysis that follows is that the bias explanation and the predictor bias are not equivalent.
In particular, the bias impact of the predictor $X_i$ in the model can be attributed to the following factors:
\begin{itemize}
  \item [(c1)] structure of the model,  
  \item [(c2)] bias in predictors,
  \item [(c3)] shape of the predictor distribution in the context of interaction.
\end{itemize}

To illustrate how the above components affect the bias contribution consider the following examples. First, for simplicity we assume that $X$ are independent and $\A=\{\Omega\}$.

Consider a model with $X \in [0,1]^2$  and $f(X)=a_1 X_1+a_2 X_2$. Let us keep $a,b$ fixed and consider the situation in which the distributions of $X_1$  and $X_2$  are fixed as well. Suppose that $X_2$  is independent of $G$ while $X_1$  and $G$ are dependent and therefore $X_1$ is biased. In that case, the $W_1$-based bias of $X_1$ will be amplified by the coefficient $a_1$, in light of the scaling property  $\Bias_{W_1}(a_1X_1|G)=a_1\Bias_{W_1}(X_1|G)$, which illustrates (c1).

Next, suppose that both $X_1$  and $X_2$ are biased and that $\Bias_{W_1} (X_1|G)\gg\Bias_{W_1}(X_2|G)$. In this case, for a model where  $a_1\ll a_2$   the bias explanation of $X_1$ will be small while the one from $X_2$ large. This implies that weakly biased predictors can have a significant impact on the model bias while strongly biased predictors may not contribute at all. At the same time, if $a_1 \sim a_2$ then the predictor with the larger bias will contribute more, which illustrates (c2).

The next example illustrates that the shape of the distribution of an unbiased predictor that interacts with a biased one affects the model bias.  Let $X \in \RR^2$  and suppose that $X_1$  is independent of $G$ while $X_2$ is biased. Consider a model with interactions $f(X)=X_1 X_2$  and suppose that $X_1$  is independent of $G$. Note that the change in the distribution of $X_1$ will affect the change in the $W_1$-based model bias through the interaction with $X_2$, which illustrates (c3).

We next demonstrate that statistical dependence and fairness are related, but not equivalent. First, note if the predictor is independent of the protected attribute, it will be fair in any fairness metric. On the other hand, a predictor can still be fair even under the presence of strong dependencies with the protected attribute, which implies that independence is a stronger notion to fairness.

Consider the model $f(X)=1_{\{X>0\}}$ and $X$ satisfying the property $P(X=\eps|G=0)=1$ and $P(X=-\eps|G=1)=1$. Note that as $\eps \to 0$, the  $W_1$-based bias of $X$ goes to zero, while KS-based bias is $1$ for all $\eps>0$.  Meanwhile, the model bias (in both metrics) remains $1$.

In the above example, the KS-metric exhibits a clear separation between the two subpopulations of the predictor and detects dependence between $X$ and $G$, ignoring the geometry of subpopulation distributions regardless of $\eps$. $W_1$ metric, on the other hand, assesses fairness at the level of values and senses the change in predictor distribution as $\eps \to 0$. 

The above example depicts a striking behavior. While the predictor bias goes to zero in the $W_1$ metric, the model bias does not. To understand when the model bias is continuous with respect to the predictor bias, we provide the following lemma:

  



\begin{lemma}\label{lmm::vanishbias}
Suppose $X_{\eps}\to X$ in $L^1(\PP)$ as $\eps \to 0^+$. Let $U \subset \RR^n$ be a common support of $\{\P_{X_{\eps}}\}_{\eps>0}$ and $P_X$. Let $G \in \{0,1\}$ be a protected attribute with $\PP(G=k)>0$, $k\in\{0,1\}$, and $\A$ as in Definition \ref{def::parity}. Suppose $f$ is continuous and bounded on $U$ and $\Bias_{W_1,\A}^{(w)}(X_{\eps}|G) \to 0$ as $\eps \to 0^+$. Then 
\[
\lim_{\eps \to 0^+}\Bias_{W_1,\A}^{(w)}(f|X_{\eps},G) = 0.
\]
Furthermore, $X$ is independent of $G$ and $\Bias_{W_1,\A}^{(w)}(f|X,G) = 0$.
\end{lemma}

\begin{proof}
See Appendix \ref{app::auxlemmas}.
\end{proof}


\subsection{Relation between bias explanations and model bias}

Bias explanations quantify the impact of the predictor on the bias in the output. More importantly, they capture different aspects of the impact factors described in (c1)-(c3) depending on the model explainer being used. For example, bias explanations that use a marginal approach (for example, the ones based on PDPs, or expected individual bias explanations (IBEs) which we introduce in Appendix \ref{sec::ibes}) capture the propagation of the predictor bias through the model, isolating this predictor from its complement in the process, which guarantees that unbiased predictors have no bias impact. In particular,  we have the following lemma.

\begin{lemma}
Let $X,f,\A$ be as in Definition \ref{def::modbias}. Let $\{\beta_i\}_{i=1}^n$ be either bias explanations of $(X,f)$ in the $(W_1,\A,w)$-metric constructed  via explainers $E_i=E_i(X_i)$  depending explicitly only on $X_i$ or expected IBEs. Suppose $\{X_i|A_m\}_{m=1}^M$ is independent of $G$ for some $i \in N$. Then $\beta_i=\beta_i^{\pm}=0$.
\end{lemma}
\begin{proof}
If $E_i = E_i(X_i)$ then it explicitly depends only on $X_i$. Having $X_i|A_m$ independent of $G$ implies $P_{X_i|\{G=0,A_m\}}=P_{X_i|\{G=1,A_m\}}$
and hence 
\[
P_{E_{i}(X_i)|\{G=0,A_m\}} = P_{X_i|\{G=0,A_m\}} \circ E_i^{-1}=P_{X_i|\{G=1,A_m\}} \circ E_i^{-1}=P_{E_{i}(X_i)|\{G=1,A_m\}}
\]
Since $m$ in the above equality is arbitrary, we obtain $\beta_i=\beta_{\pm}=0$. The proof for expected IBEs is similar.
\end{proof}

Game theoretical model explainers, on the other hand, produce explanations that depend on the joint distribution $X$ which allow one to take into account interactions and dependencies. For this reason the bias explanations based on Shapley explainers can capture the effect of all three factors (c1)-(c3) to the model bias. 

For example, let $f(X)=X_1 X_2$  with $\E[X]=0$ and assume the predictors are independent. Suppose that $X_1$ is unbiased and $X_2$ is biased. Computing Shapley explanations we obtain:
\[
\varphi_i(X;f,v)=\frac{1}{2} X_1 X_2, \quad v \in \{\vce,\vpdp\}.
\]
The above equality indicates that the corresponding bias explanations for both predictors coincide regardless of the predictors' bias. Thus, $X_1$  contributes to the model bias by interacting with the biased predictor via scaling.

 Another interesting question to ask is whether bias explanations being all zero guarantees that the model is unbiased. As it turns out, it is not always true and depends on the model explainer being used. However, for bias explanations based on marginal and conditional Shapley we have the following lemma:

\begin{lemma}
Let $X \in \RR^n$ be predictors, $f$ a model, and $E_i=\varphi_i(X;f,v)$, with $v \in\{\vce,\vpdp\}$. Let $\beta=\{\beta_i\}_{i=1}^n$ be bias explanations based on  predictor explainer $E_i$ in the $(W_1,\A,w)$-metric. If $|\beta|=0$, then $\Bias_{W_1,\A}^{(w)}(f|X,G)=0$.
\end{lemma} 
\begin{proof}
Suppose $\beta_i=0$ for all $i \in \{1,2,\dots,n\}$. Then for each $i$ and $m$, $\varphi_i(X;f,v)|A_m$ is independent of $G$. Then by the efficiency property we conclude that $\sum_{i=1}^n \varphi_i(X;f,v)+\E[f(X)]=f(X)$ conditioned on $A_m$ is independent of $G$. 
\end{proof}

For bias explanations based on Shapley values, it might be difficult to determine which factor among (c1)-(c3) is responsible for the bias impact. In particular, non-zero bias explanation of an unbiased predictor can occur via interaction with biased predictors.


\paragraph{Analogy with variance.}

To understand that bias explanations are not the only type of explanations where the shape of the predictor distribution and interactions play a role in the impact, consider the following analogy: suppose we are trying to estimate the contribution of the predictor $X_i$  to the variance of the model $f(X)$. As with bias explanations, the variance of the predictor $X_i$  is not a proxy for the impact on the variance of $f$. The predictor might not be explicitly used by the model structure, hence its variance may play no role at all. While interacting with other predictors, the distribution of the predictor has a direct effect on the model variance as well, even if the variance of the predictor is small. Similarly, the model structure can ensure that the model variance can be made arbitrarily large no matter how small the predictor variance is.







%% file: bias_mitigation.tex
\section{Bias mitigation for Wasserstein-based fairness metrics}\label{sec::bias_mitigation}

  In light of regulatory restrictions on use of the protected attribute, we consider the following fairness assessment and bias mitigation procedure:

\begin{itemize}
  \item [(S1)] Bias measurement and interpretability. Given a model $f$ perform the fairness assessment by measuring the bias among sub-population distributions and determine the main drivers of that bias, that is, the list of predictors contributing the most to that bias.

  \item [(S2)] Mitigation. Given a trained model $f(x)$, construct a post-processed model utilizing the information on the main drivers, and without the direct use of the protected attribute.
\end{itemize}

In this article we address (S2), which relies on (S1) investigated in the companion paper \citet{Miroshnikov2020}. Utilizing bias attributions, we are able to reduce the dimensionality of the bias mitigation problem and present several post-processing approaches for bias mitigation. 

In what follows, we primarily work with the $W_1$-metric, but all results and conclusions can be trivially extended to the $(W_1,\A,w)$ -metric.

\subsection{Selection of main bias drivers} \label{sec::biasdriverselect}

\subsubsection{Model bias as a superposition of bias explanations} \label{sec::modsuperpos}

  Before outlining our approaches for bias mitigation, we present in what follows a connection between the model bias and bias explanations. This connection is a crucial component in forming our approaches.

Recall that the model bias can be decomposed into the positive and negative components 
\[
  \Bias_{W_1}(f|G) = \Bias_{W_1}^+(f|G) + \Bias_{W_1}^-(f|G).
\]

We now express the positive and negative model biases using Shapley bias explanations as follows: 
\[
\begin{aligned}
\Bias^+_{W_1}(f|X,G) 
&= \sum_{i} \1_{\{I_{++}\}}\varphi_i[v^{bias+}] - \sum_{i}\1_{\{I_{+-}\}}(-\varphi_i[v^{bias+}])\\
\Bias^-_{W_1}(f|X,G) 
&= \sum_{i} \1_{\{I_{-+}\}}\varphi_i[v^{bias-}] - \sum_{i}\1_{\{I_{--}\}}(-\varphi_i[v^{bias-}])
\end{aligned}
\]
where $I_{\pm+}=\{i: \varphi_i[v^{bias\pm}]>0\}$ and $I_{\pm-}=\{i: \varphi_i[v^{bias\pm}]<0\}$.

We next define the following non-negative attributions:
\begin{equation}\label{shapbiasatom}
\begin{aligned}
\bar{\beta}_i^{++}&=\1_{\{I_{++}\}}\varphi_i[v^{bias+}],& \bar{\beta}_i^{+-}&=\1_{\{I_{+-}\}}(-\varphi_i[v^{bias+}])\\
\bar{\beta}_i^{-+}&=\1_{\{I_{-+}\}}\varphi_i[v^{bias-}],& \bar{\beta}_i^{--}&=\1_{\{I_{--}\}}(-\varphi_i[v^{bias-}])
\end{aligned}
\end{equation}
and conclude that the model bias is a superposition of four types of attributions: 
\begin{equation}\label{modsuperposition}
\Bias_{W_1}(f|X,G) =  \sum_{i \in N} \bar{\beta}_i^{++} + \sum_{i \in N}\bar{\beta}_i^{-+} - \sum_{i \in N}\bar{\beta}_i^{+-}  - \sum_{i \in N}\bar{\beta}_i^{--} \geq 0.
\end{equation}
Here  $\beta_i^{++}, \ \beta_i^{-+}$ represent the attribution of predictor $X_i$ to the increase of the positive and negative flows between the model subpopulations, respectively. Similarly,  $\beta_i^{+-}, \ \beta_i^{--}$ represent the attribution  of predictor $X_i$ to the reduction in the positive and negative model flows, respectively.

We next define the attributions
\begin{equation}\label{shapbiasposneg}
\bar{\beta}_i^{+} = \bar{\beta}_i^{++} + \bar{\beta}^{--}, \quad \bar{\beta}_i^{-} = \bar{\beta}_i^{-+} + \bar{\beta}_i^{+-}, \quad \bar{\beta}_i = \bar{\beta}_i^{+} + \bar{\beta}_i^{-}
\end{equation}
and note that
\begin{equation}\label{modsuperpositionnet}
\Bias_{W_1}^{net}(f|G) =  \sum_{i \in N}\bar{\beta}_i^+ - \bar{\beta}_i^-.
\end{equation}
Here $\bar{\beta}_i^{+}$ represents the contribution of the predictor to the transport effort of pushing the non-protected class in the favorable direction, which includes both the contribution to increasing the positive model bias and decreasing the negative model bias; the latter occurs in cases where the predictor interacts with other predictors through the model structure. A similar description holds for $\bar{\beta}_i^{-}$.

\subsubsection{Selection procedure} 


We first discuss the selection in the context of Shapley bias explanations in view of their additivity. Pick two thresholds $\eps_+,\eps_->0$, and form the lists 
\[
N_+=\{i \in N: \bar{\beta}_i^+>\eps_+\}, \quad N_-=\{i \in N: \bar{\beta}_i^->\eps_-\}.
\] 
If possible, the thresholds should be chosen so that neither list is empty. Set the list of most impactful predictors to be
\begin{equation}\label{impactlist}
M=N_+ \cup N_- 
\end{equation}
and note that 
\begin{equation}\label{modbiastrunc}
\Bias_{W_1}(f|X,G) =  \sum_{i \in M} \big(\bar{\beta}_i^{++} + \bar{\beta}_i^{-+} - \bar{\beta}_i^{+-}  - \bar{\beta}_i^{--} \big) + O(\eps)\geq 0.
\end{equation}
It is also practical to partition the list M into three disjoint lists:
\begin{equation}\label{Mpartition}
\begin{aligned}
  M_+ = \{i\in M:\bar{\beta}_i^+ \gg \bar{\beta}_i^-\}, \,\, M_- = \{i\in M: \bar{\beta}_i^+ \ll \bar{\beta}_i^-\}, \,\, M_{0} = \{i \in M: \bar{\beta}_i^+ \sim \beta_i^-\}.
\end{aligned}
\end{equation}

 The objective of the above procedure is to reduce the problem dimensionality so that $m=|M|\ll n$. If the list $M$ is too large, one may reduce it as follows: given $m_*$, rank order the list $N_+$  and the list $N_-$  and pick the largest $m_*$ in each list, in which case $m=|M|\leq 2m_*$. While \eqref{modbiastrunc} will no longer hold, it is not relevant to our method as we only need to select the most impactful predictors from each list.

The Shapley bias explanations have very high complexity. For this reason, in practice, we compute the basic bias explanations $\{\beta_i^{\pm}\}$ defined in \eqref{biasexpl} based on model explainers such as PDPs or the marginal Shapley values. The quantities $\{\beta_i^{\pm}\}$, roughly speaking, serve as approximants to the quantities $\{\bar{\beta}_i^{\pm}\}$ defined in \eqref{shapbiasposneg}. The selection procedure is carried out as discussed above.

\subsection{Compressive mappings}

It has been observed by \cite{Feldman2015} that if a dataset is fair, with respect to a given fairness metric, then a classifier trained on such data will be fair. This stays true for regressors as well; see Lemma \ref{lmm::vanishbias}. Since most datasets are not fair, \cite{Feldman2015} suggested repairing and partial repairing of the dataset: given $(X,G)$, form $\bar{X}=\bar{X}(X,G)$ via predictor transformations using median distributions which arise in transport theory.

 In the aforementioned work, repairing the dataset requires explicit knowledge of the protected attribute $G$, which in practice is not available either in training or prediction stage and, as stated, is not allowed by regulations.  Furthermore, the transformations are metric-specific. This means that if the fairness metric or fairness penalization (in reference to partial repair) changes, the repair method and retraining would need to be re-applied, leading to high computational cost in the case of large datasets. 

Motivated by \cite{Feldman2015}, we see potential in outlining a different approach that utilizes predictor transformations, without changing datasets themselves and avoiding retraining. First, we reduce the problem's dimensionality using the bias explanations as discussed in Section \ref{sec::biasdriverselect} and obtain the list $M$ of most bias impactful predictors, with $m\ll n$. Next, we focus on predictors $X_M$. In principle, one can partially repair the set $X_M$ and then use it as an input to the trained model. While this seems appealing, this still requires knowledge of $G$. For this reason, we take another route. Motivated by the discussion in Section \ref{sec::biasdriverselect}, we construct appropriate transformations of predictors in $X_M$  that allow us to adjust their bias explanations without explicitly using $G$, leading to change in the model bias.

To provide some intuition behind the choice of transformations, consider a model $f$ such that $\Bias_{W_1}^+(f|X,G)>0$ and $\Bias_{W_1}^-(f|X,G) = 0$. In that case, by \eqref{modsuperposition} we have
\begin{equation}\label{superpospos}
\Bias_{W_1}(f|G) = \Bias_{W_1}^+(f|G) = \sum_{i \in N}\bar{\beta}_i^+  - \sum_{i \in N} \bar{\beta}_i^- = A_+ - A_- \geq 0.
\end{equation}

Hence adjusting predictors so that $A+$ decreases and $A-$ increases would lead to the reduction of model bias. The adjustments can be limited to the list $M$ to reduce complexity.

Note that the Wasserstein-based metric, defined on a normed state space, satisfies the scaling property 
\begin{equation*}
\Bias_{W_1}^{\pm}(a \cdot X_i + b|G) = a \cdot  \Bias_{W_1}^{\pm}(X_i|G), \quad a,b,\in \RR.
\end{equation*}
Thus, any type of transformation that rescales the predictor will adjust its bias and hence its bias impact. In particular, for compressive transformations that pull the values of a predictor towards a reference point, the basic bias explanations $\beta_i^+$, $\beta_i^-$ of the adjusted predictor $X_i$ are expected to decrease. For expansive maps, that push values away from a reference point, the bias explanations are expected to increase. Therefore, one possible strategy in the scenario \eqref{superpospos} is to compress predictors in $M_+$  and expand predictors in $M_-$, while predictors in $M_0$ can be transformed by compressing and expanding different regions of the predictor values in a way so that $\beta_i^+$  decreases and $\beta_i^-$ increases.


Motivated by the above discussion, we introduce the following transformations:
\begin{definition}\label{def::comprprop}
Let $\{\bar{T}(z; \alpha, z_*)\}$ with $z, z^* \in \RR^m$, $\alpha\in \RR^{km}$, be a family of continuous maps in the form 
\begin{equation}\label{comptransf}
\bar{T}(z;\alpha,z^*)=(T(z_1; a^1, z_{1}^*),T(z_2; a^2, z_{2}^*),\dots,T(z_m; a^m, z_{m}^*)), 
\end{equation}
where $\alpha=(a^1,a^2,\dots,a^m)$, and $t \to T(t; a, t^*)$, with $a \in \RR^k, t_*\in \RR$, is strictly increasing and satisfies
\begin{itemize}
  \item [$(i)$] $\displaystyle \lim_{a \to +\infty}T(t; a, t^*)=t^*$ such that $|T(t;a,t)-t_*|\leq C |t-t^*|^{q} a^{-\delta}$ for $q \geq 0$, $\delta > 0$.\\
  \item [$(ii)$] $\displaystyle \lim_{a\to a_0} T(t;a,t^*)=t$, for some $a_0=a_0(t^*) \in \RR^k$.\\
  \item [$(iii)$] $\displaystyle \lim_{a \to 0+} \text{sgn}(t-t^*)T(t;a,t^*)=+\infty$.
\end{itemize}
We call the above family a compressive family.
\end{definition}


The next lemma states that compressing biased predictors to a point removes the bias in the model.
\begin{lemma}\label{lmm::partialneutrmod}
Let $X$, $G$, $f$, $\mathcal{A}$ be as in Definition \ref{def::modbias}. Let $M \subset N$. Suppose $X_{-M}$ is unbiased in the $(W_1,\A,w)$-metric and $f$ is continuous on the support of $X$. Let $\{\bar{T}(\cdot;\alpha,x^{*}_M)\}$ be a compressive family on $\RR^{m}$. Let an intermediate post-processed model be given by
\begin{equation}\label{partialneutrmod}
\tilde{f}(X; \alpha,x_M^*):=f(\bar{T}(X_M;\alpha,x^*_M),X_{-M}), \quad \alpha \in \RR^{km}, \, x_M^* \in \RR^m.
\end{equation}
Then for any continuous, monotonic function $C:\RR \to \RR$ we have 
\begin{equation}\label{eq::limbias}
\lim_{\alpha \to \infty} \Bias_{W_1,\mathcal{A}}^{(w)}( C\circ\tilde{f}(X;\alpha,x_M^*)|X,G) = \Bias_{W_1,\mathcal{A}}^{(w)}(C\circ f(x_{M}^*,X_{-M})|X_{-M},G)=0
\end{equation}
provided $\E\big[|X_M|^{max(1,q)}\big]<\infty$ for $q>0$ in Definition \ref{def::comprprop}(i).
\end{lemma}

\begin{proof}
The result follows directly from Lemmas \ref{lmm::vanishbias} and Lemma \ref{lmm::comprbias}; the latter in the appendix.
\end{proof}

In Lemma \ref{lmm::partialneutrmod} the assumption that $f$ is continuous doesn't affect the second equality in \eqref{eq::limbias}. In principle, the continuity of $f$ can be dropped and replaced with the assumption that the true regressor is continuous. In this case, the second equality in \eqref{eq::limbias} will be modified to contain the error between the true regressor and $f$. The map $C$ in the above lemma is a calibrating map that we obtain by implementing monotonic regression discussed in Section \ref{sec::calib}.

\begin{remark}\rm  
The reason for the type of transformation in \eqref{comptransf} is the desire to conserve the rank-ordering of data samples in each individual direction.
\end{remark}

\subsection{Minimization problem for bias mitigation}\label{sec::bayesminim}

Equation \eqref{partialneutrmod} illustrates that compressing a collection of predictors $X_M$ to a fixed point $x_M^*$ neutralizes the action of those predictors on the model bias, leaving the explicit dependence only on the unbiased ones $X_{-M}$. This produces an unbiased post-processed model. However, compressing most bias-impactful predictors to a fixed point may significantly impact the model performance. For this reason, considering the relationship in \eqref{modsuperposition}, we design an approach that takes advantage of bias offsetting by reshaping predictors in $X_M$ according to their bias explanations, which allow for decrease in the model bias.

\begin{figure}
\centering
  \begin{subfigure}[t]{0.3\textwidth}
    \centering
    \includegraphics[width=\textwidth]{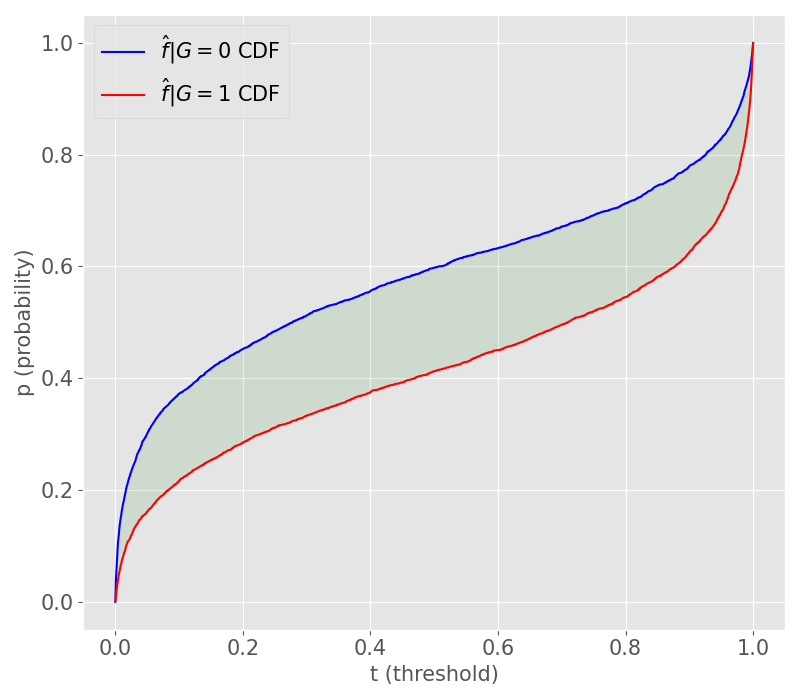}
    \caption{\footnotesize Subpopulation distributions}\label{fig::subpop_distr_m1}
  \end{subfigure}
~~
  \begin{subfigure}[t]{0.3\textwidth}
    \centering
    \includegraphics[width=\textwidth]{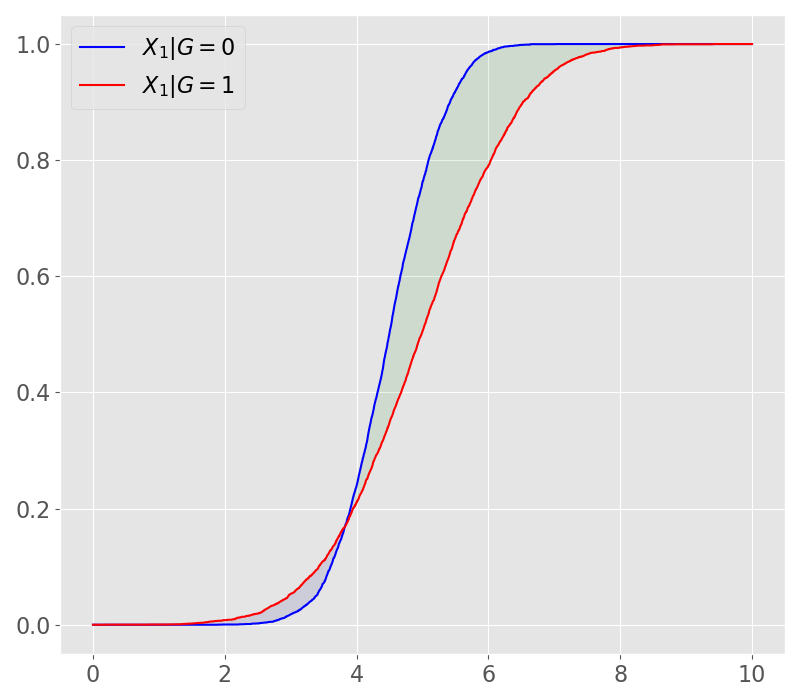}
    \caption{\footnotesize $X_1$ CDFs}\label{fig::subpop_x1_m1}
  \end{subfigure}
~~
  \begin{subfigure}[t]{0.3\textwidth}
    \centering
    \includegraphics[width=\textwidth]{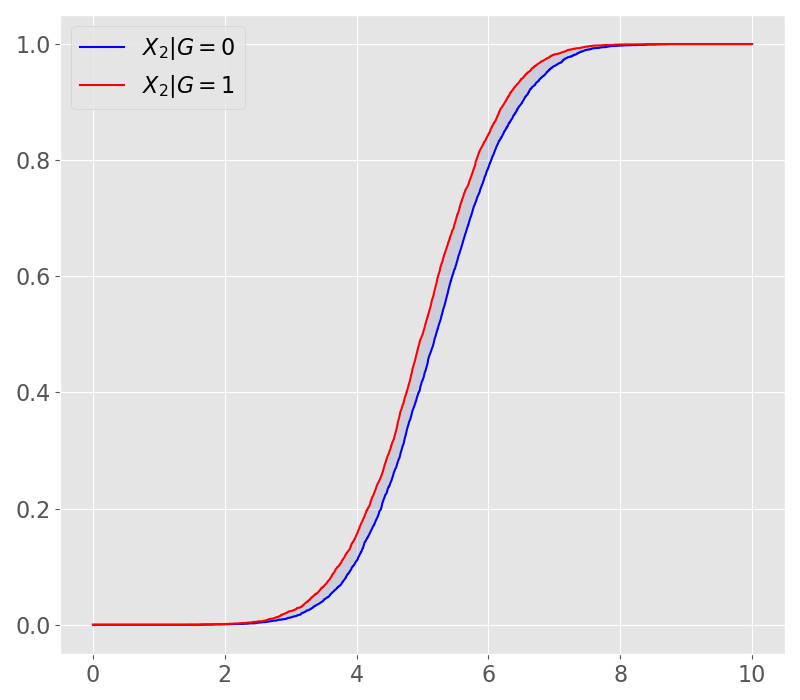}
    \caption{\footnotesize $X_2$ CDFs}\label{fig::subpop_x2_m1}
  \end{subfigure}
 ~~
  \begin{subfigure}[t]{0.3\textwidth}
    \centering
    \includegraphics[width=\textwidth]{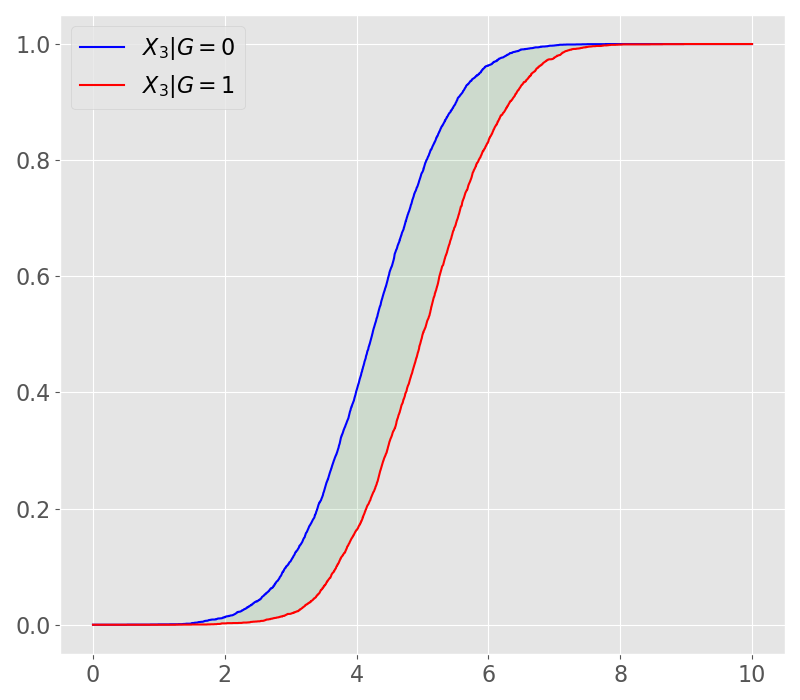}
    \caption{\footnotesize $X_3$ CDFs}\label{fig::subpop_x3_m1}
  \end{subfigure}
 ~~
  \begin{subfigure}[t]{0.3\textwidth}
    \centering
    \includegraphics[width=\textwidth]{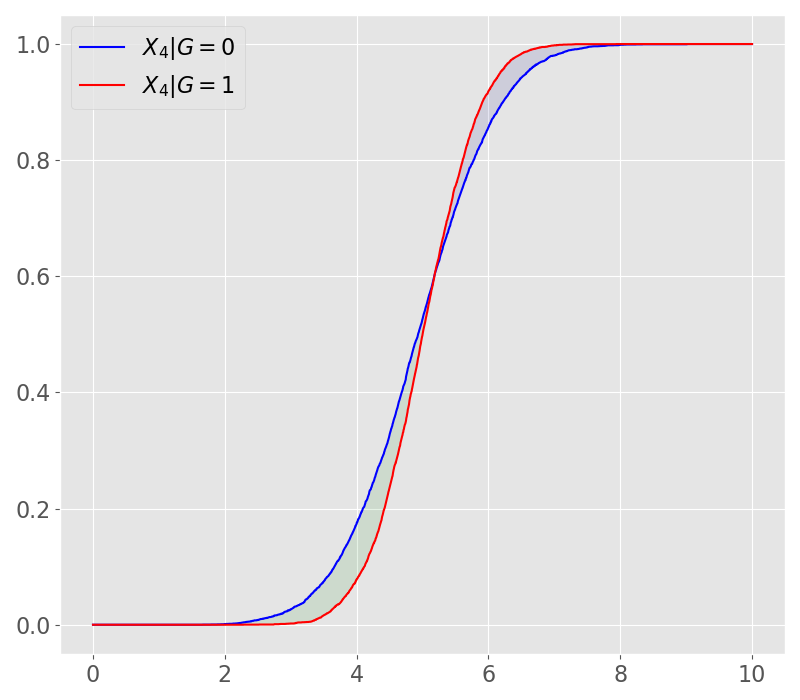}
    \caption{\footnotesize $X_4$ CDFs}\label{fig::subpop_x4_m1}
  \end{subfigure}
~~
  \begin{subfigure}[t]{0.3\textwidth}
    \centering
    \includegraphics[width=\textwidth]{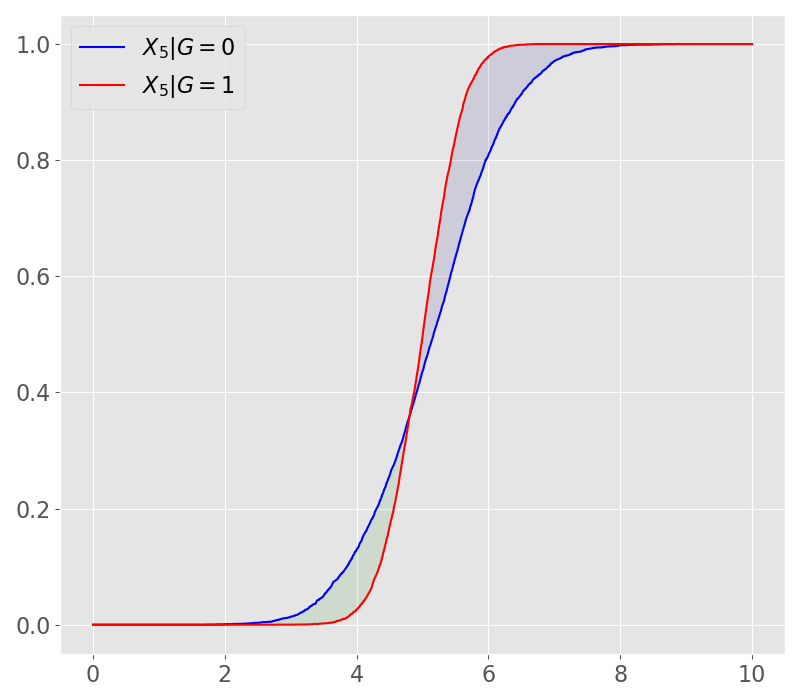}
    \caption{\footnotesize $X_5$ CDFs}\label{fig::subpop_x5_m1}
  \end{subfigure}
~~
  \begin{subfigure}[t]{0.3\textwidth}
    \centering
    \includegraphics[width=\textwidth]{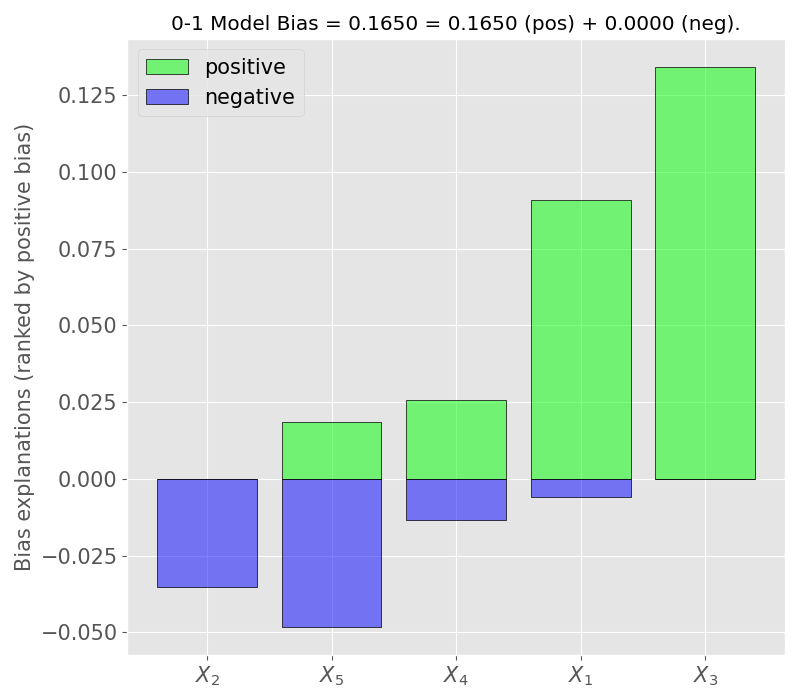}
    \caption{\footnotesize Bias explanations}\label{fig::bias_expl}
  \end{subfigure}
~~
  \begin{subfigure}[t]{0.3\textwidth}
    \centering
    \includegraphics[width=\textwidth]{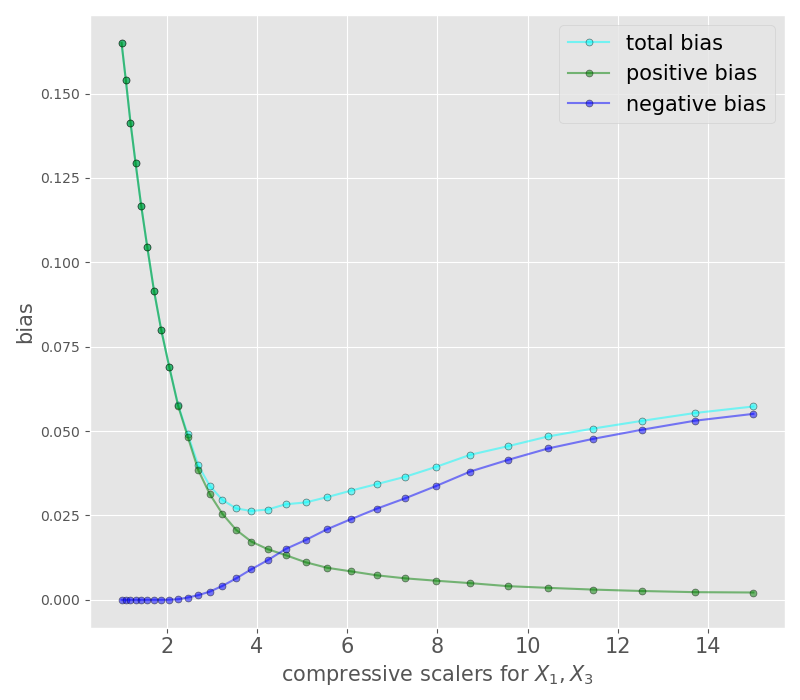}
    \caption{\footnotesize Change in bias}\label{fig::bias_change_compr}
  \end{subfigure}
 \caption{ Bias information for model \eqref{realisticmod}.} 
\end{figure}

To understand how bias offsetting helps in reducing the bias in the model output, consider the following data generating model:
\begin{equation}\label{realisticmod}\tag{M1}
\begin{aligned}
&\mu=5, \quad a=\tfrac{1}{20}(10,-4,16,1,-3)\\
&X_1 \sim N(\mu-a_1 (1-G), 0.5+G ), \quad X_2 \sim N(\mu-a_2 (1-G), 1 ) \\
&X_3 \sim N(\mu-a_3 (1-G), 1 ),     \quad X_4 \sim N(\mu-a_4 (1-G), 1-0.5 G ) \\
&X_5 \sim N(\mu-a_5 (1-G),1-0.75 G ) \\
&Y \sim Bernoulli(f(X)), \quad f(X)=\P(Y=1|X)={logistic(2 \big(\textstyle{\sum_{i}} X_i-24.5)\big)}.
\end{aligned}
\end{equation}

\begin{algorithm}
\SetAlgoLined 
\KwData{Model $\fhat$, training or holdout set $(X,Y)$, test set $(\bar{X},\bar{Y})$, the set $M$ of bias-impactful predictors.
}
 \KwResult{Models $\{\bar{f}(\cdot; \gamma) \in \family(\fhat), \gamma=(\alpha,x^*)\}$ constituting the efficient frontier of $\mathcal{F}$ in \eqref{postprocfam}. } 

{\bf Initialization parameters:} the number $n_{prior}$ of random points  $\gamma = (\alpha, x_{M}^*)$, the prior $P_{prior}(d\gamma)$, fairness penalization parameters $\omega = \{ \omega_1, \dots, \omega_J \}$, the number $n_{bo}$ of Bayesian steps  for each $\omega_j$.\\
Sample $\{\gamma_i\}_{i=1}^{n_{prior}}$ from $P_{prior}(d\gamma)$\\
\For{$i$ in $\{1,\dots,n_{prior}\}$}
{
$loss(\gamma_i;X,Y):=\E[\mathcal{L}(X,\bar{f}(X;\gamma_i))]$, $\bar{f} \in \family(\fhat)$.\\
$bias(\gamma_i;X):=\Bias_{W_1,\A}(\bar{f}(X;\gamma_i)|G)$.
}

\For{$j$ in $\{1,\dots,J\}$}
{ 
\For{$i$ in $\{1,\dots,n_{prior}\}$}
{
 $L(\gamma_i,\omega_j) := loss(\gamma_i; X,Y) + \omega_j \cdot bias(\gamma_i;X)$
}
Pass $\{ \gamma_i, L(\gamma_i,\omega_j)\}_{i=1}^{n_{prior}}$ to the Bayesian optimizer that seeks to minimize $L(\cdot,\omega_j)$.

Perform $n_{bo}$ iterations of Bayesian optimization which produces $\{\gamma_{t,j} \}_{t=1}^{n_{bo}}$.
}
Compute $(\gamma, bias(\gamma;\bar{X}), loss(\gamma;\bar{X},\bar{Y}))$ for $\gamma \in \{\gamma_i\} \cup \{\gamma_{t,j}\}$, giving a collection $\mathcal{V}$.\\
 Compute the convex envelope of $\mathcal{V}$ and exclude the points that are not on the efficient frontier.
 \caption{Efficient frontier reconstruction}\label{BHPSalgo}
\end{algorithm}

In \eqref{realisticmod}, the bias explanations for predictors contain both positive and negative components; see Figure \ref{fig::bias_expl}. The predictors with positive bias explanations dominate those with negative bias explanations, which results in negative model bias being equal to zero as in \eqref{superpospos}, illustrating the bias offsetting taking place; see Figure \ref{fig::subpop_distr_m1}-\ref{fig::subpop_x5_m1}. For instructive purposes, let us illustrate the effect of rescaling predictors $X_1$, $X_3$ on the model bias by compressing the two predictors towards their means. For this, consider the trained model perturbation
\[
  a \to f_{a}(X) = f(T(X_1;a,x_1^*), X_2, T(X_3;a,x_3^*), X_4, X_5)
\]
with $a\in[1,15]$ and $T$ as in Definition \ref{def::comprprop}. In view of offsetting, for low level of compression the positive model bias decreases while the negative stays constant. As we compress further,  the negatively biased predictors become prevalent because they were not adjusted, leading to an increase in the negative model bias while the positive model bias tends to zero. This produces a U-shaped curve of the model bias as a function of compression.

The pedagogical example above motivates us to setup the following minimization problem. Given the Wasserstein-based $(W_1,\A,w)$-metric, the list M of most impactful predictors in that metric, and a compressive family as in Definition \ref{def::comprprop}, define a family of post-processed models
\begin{equation}\label{postprocfam}
\begin{aligned}
\mathcal{F}(f)= \Big\{\bar{f}(X; \alpha, x_M^*) = \mathcal{C}[f(\bar{T}(X_M;\alpha,x^*_M),X_{-M});\,X,f], \, (\alpha,x_M^*) \in (A,\Chi_M)  \Big\},
\end{aligned}
\end{equation}
where $A\subset \RR^{km}$, $\Chi_M \subset {\rm supp}(P_{X_M})$ and $\mathcal{C}$ is a calibrating operator. 


Constructing the efficient frontier of the family $\family$ amounts to solving a constrained minimization problem.  This problem can be reformulated in terms of generalized Lagrange multipliers using the Karush-Kuhn-Tucker approach
\citep{Karush,KuhnTucker} as follows 
\begin{equation}\label{minfront}\tag{BM}
    \alpha_{*}(\omega) = \text{arg}\min_{\bar{f} \in \mathcal{F}(f)} \Big\{ \E[\mathcal{L}(Y,\bar{f}(X)] + \omega \cdot \Bias_{W_1,\A}^{(w)}(\bar{f}|G,X) \Big\}, \quad \omega \ge 0.
\end{equation}
To solve \eqref{minfront}, we use the lower dimensionality to our advantage and construct the efficient frontier via Bayesian optimization; see Algorithm \ref{BHPSalgo}. We avoid using gradient descent techniques to accommodate non-smooth machine learning models such as tree-based models.

\paragraph{On algorithm complexity.} In Algorithm \ref{BHPSalgo}, one could replace the training set $(X,Y)$ with a new holdout set $(X^*,Y^*)$, with size $K^*$, which can be chosen to be smaller than the training set due to the low problem dimensionality in \eqref{minfront}. At the same time, Bayesian optimization is a learning procedure, and thus the size of $(X^*,Y^*)$ should be large enough to avoid overfitting, especially if one forms a large list $M$; see Figure \ref{fig::overfit_m1}. Specifically, when evaluating the loss and bias functions on the set of $n_{prior}$ randomly drawn parameters $\{\gamma_i\}$ in lines $3$-$10$ of the algorithm, the error of estimation of the objective function at each random parameter $\gamma_i$ is guaranteed to be $O(1/\sqrt{K^*})$, because the parameters have been drawn independently (non-adaptively) for all values using a fresh dataset. However, in line $11$ the Bayesian optimization performs a sequential search of best parameters using a Markovian process utilizing the information from the $n_{prior}$  values $\{(\gamma_i,L(\gamma_i,\cdot)\}$. This amounts to ML learning of parameters, which may allow for overfitting to occur if $K^*$ is small. Thus, $K^*$ should be chosen according to the dimensionality of the parameters in (5.15); for instance, see \citet[Section 7.4]{Hastie et al}.


Alternatively, one can treat the parameters $\gamma=(\alpha,x_M^*)$ as hyperparameters and the Bayesian steps in line $11$ as an adaptive process, during which the holdout set is accessed $n_{prior}+n_{bo}$ times, for fixed  $\omega_j$. Adaptivity is known to increase the statistical error of expectation estimation; in general, the bound on the error of the objective function estimation may be as large as $O(\sqrt{(n_{prior}+n_{bo})/K^*})$. 

The work of \citet{Dwork2015} on adaptive data analysis introduces the concept of the reusable holdout (Thresholdout), which allows one to access the same holdout set in adaptive procedures many times with much lower deterioration of the bound on the expected value of a statistic. In particular, according to Theorem 25 of \citet{Dwork2015}, if access to $(X^*,Y^*)$ is performed via a differentially private Thresholdout algorithm with a budget $B=n_{prior}+n_{bo}$, the number of samples needed for given a tolerance $\tau$ (with probability $\beta$) is estimated to be
 \begin{equation}\label{algocomplex}
K^* \geq O\bigg(\frac{\ln(\tfrac{n_{prior}+n_{bo}}{\beta})}{\tau^2} \bigg) \bigg(  n_{bo} \cdot \ln\bigg( \frac{\ln(\tfrac{n_prior+n_{bo}}{\beta})}{\tau}  \bigg)  \bigg)^{1/2}.
 \end{equation}
Thus, using the reusable holdout in the bias mitigation procedure may allow for a lower complexity compared to using the entire training set, especially if that set is large. Specifically, the complexity of Algorithm \ref{BHPSalgo} with reusable holdout is $O(K^*\cdot (n_{prior}+n_{bo})\cdot {\rm dim}(X))$.



\begin{figure}[t]
\centering
  \begin{subfigure}[t]{0.32\textwidth}
    \centering
    \includegraphics[width=\textwidth]{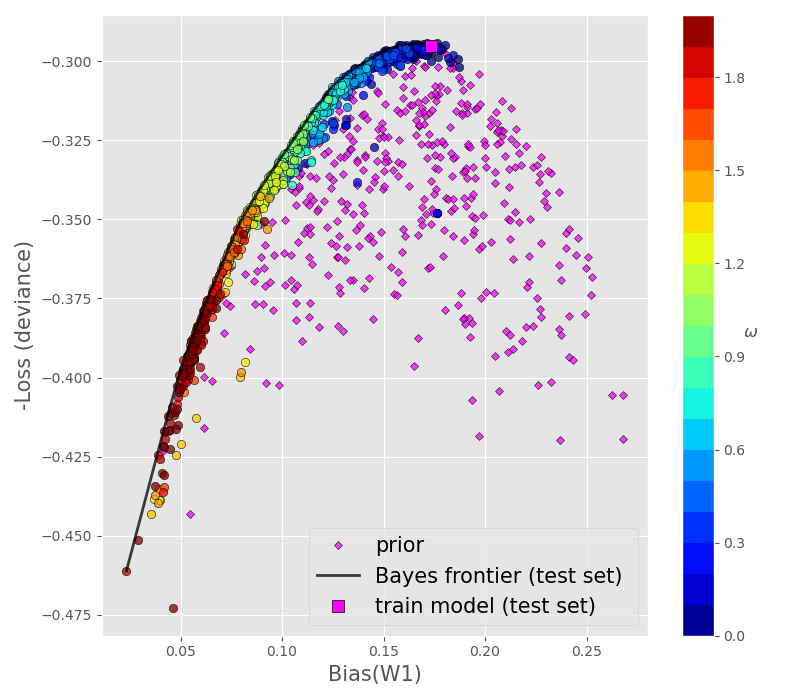}
    \caption{\footnotesize Bayesian search}\label{fig::bayes_m1}
  \end{subfigure}
\begin{subfigure}[t]{0.32\textwidth}
    \centering
    \includegraphics[width=\textwidth]{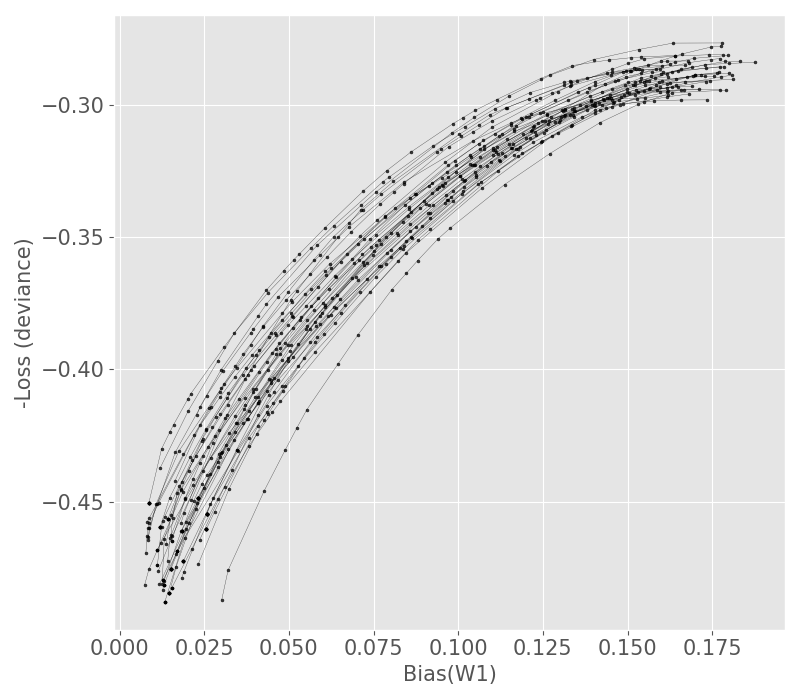}
      \caption{\footnotesize Frontiers with redrawn datasets}\label{fig::efvar_m1}
  \end{subfigure}    
  \begin{subfigure}[t]{0.32\textwidth}
    \centering
    \includegraphics[width=\textwidth]{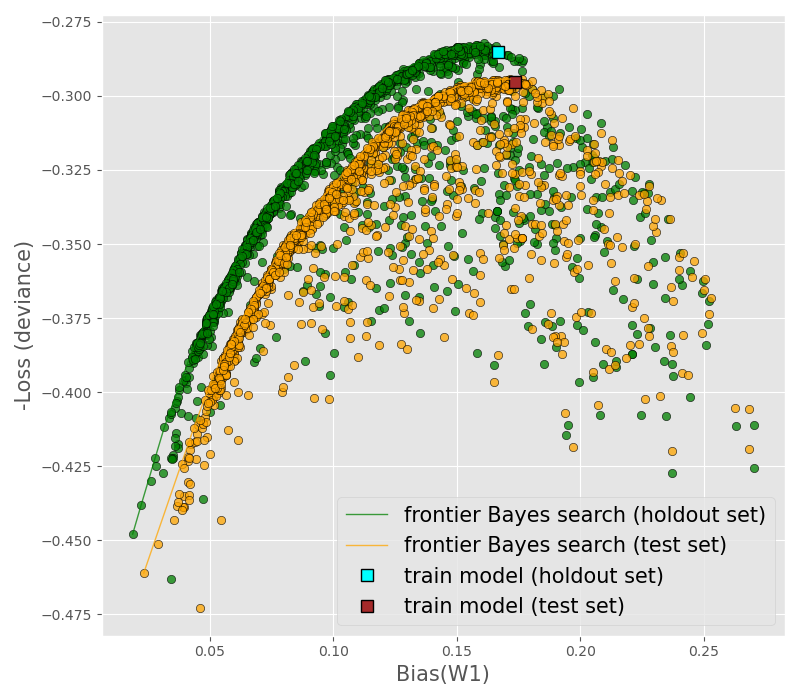}
    \caption{\footnotesize Overfitting}\label{fig::overfit_m1}
  \end{subfigure}
 \caption{ Bayesian search for Model \eqref{realisticmod}.} 
\end{figure}

\subsection{Numerical examples}

\subsubsection{Global compression}\label{sec::globcomp}

We use several transformations in this work, but the simplest one that satisfies the properties in Definition \ref{def::comprprop} is a linear map that pulls the predictor values towards a fixed focal point globally
\begin{equation}\label{GlobComp}
T(t;a,t^*) = \frac{1}{a}(t-t^*) + t^*, \quad a>0
\end{equation}
where $t^*$ is a fixed point in the support of the given predictor and $a\in \RR_+$ is a compressive scalar.

Consider the instructive model \eqref{realisticmod}. By design, all predictors but $X_4$ produce substantially large bias explanations; see Figure \ref{fig::bias_expl}.  For this reason, we set the list of predictors to be $M=\{1,2,3,5\}=\{i_1,i_2,i_3,i_4\}$ and choose to work with transformation $T$ given by \eqref{GlobComp}.   We then train a regularized GBM model with 10,000 samples and apply Algorithm~\ref{BHPSalgo}. We consider the parameters $n_{prior}=400$, $n_{bo}=50$, $\{\omega_j = 2\cdot\frac{j}{20}\}_{j=0}^{20}$, $\gamma = (\alpha, x_M^*)$ where we set $x_M^* = \E[X_M]$ to be fixed and let the parameter $\alpha=\{a_{i_k}\}_{k=1}^4$ vary with the bounds $a_{i_k} \in [0.5, 2]$. Figure \ref{fig::bayes_m1} depicts the prior as the collection of points in magenta color,  and the Bayesian iterations are depicted by the colored points where the color corresponds to different levels of $\omega$. Figure \ref{fig::efvar_m1} shows the variance of the BP efficient frontier by repeating the experiment 40 times with newly drawn dataset and retrained model. 


\begin{figure}[t]
\centering
  \begin{subfigure}[t]{0.32\textwidth}
    \centering
    \includegraphics[width=\textwidth]{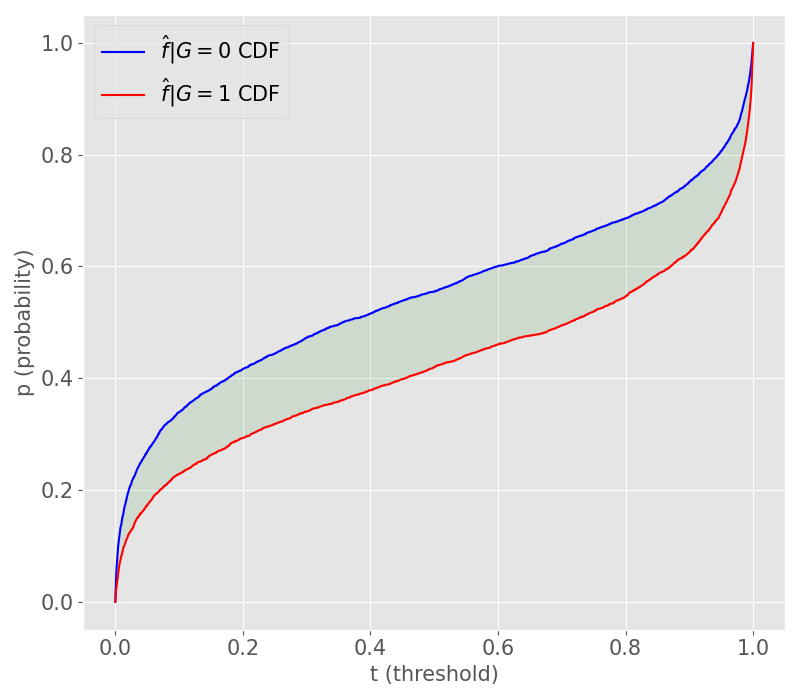}
    \caption{\footnotesize CDFs }\label{fig::subpop_distr_m2}
  \end{subfigure}
  \begin{subfigure}[t]{0.32\textwidth}
    \centering
    \includegraphics[width=\textwidth]{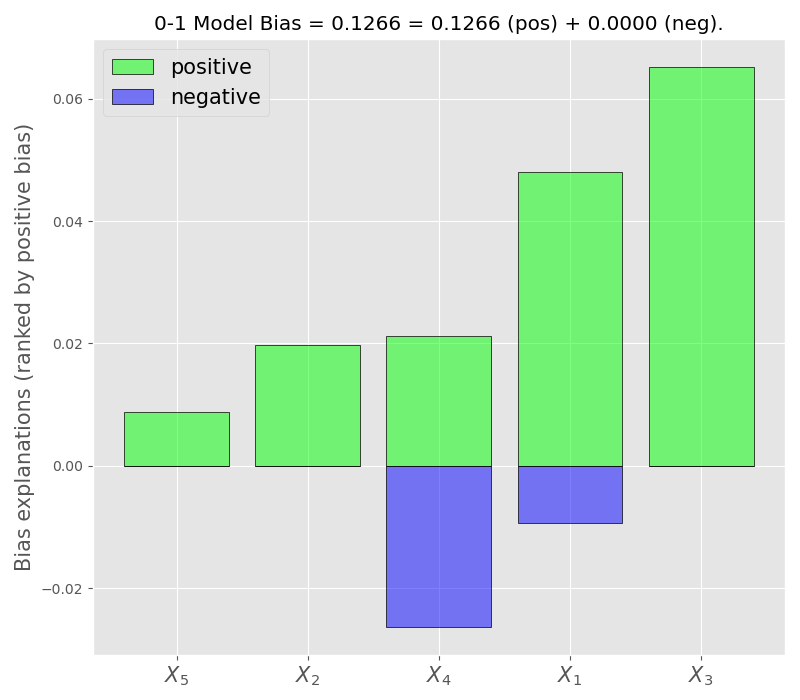}
    \caption{\footnotesize Bias explanations (train)}\label{fig::bias_expl_m2}
  \end{subfigure}
  \begin{subfigure}[t]{0.32\textwidth}
    \centering
    \includegraphics[width=\textwidth]{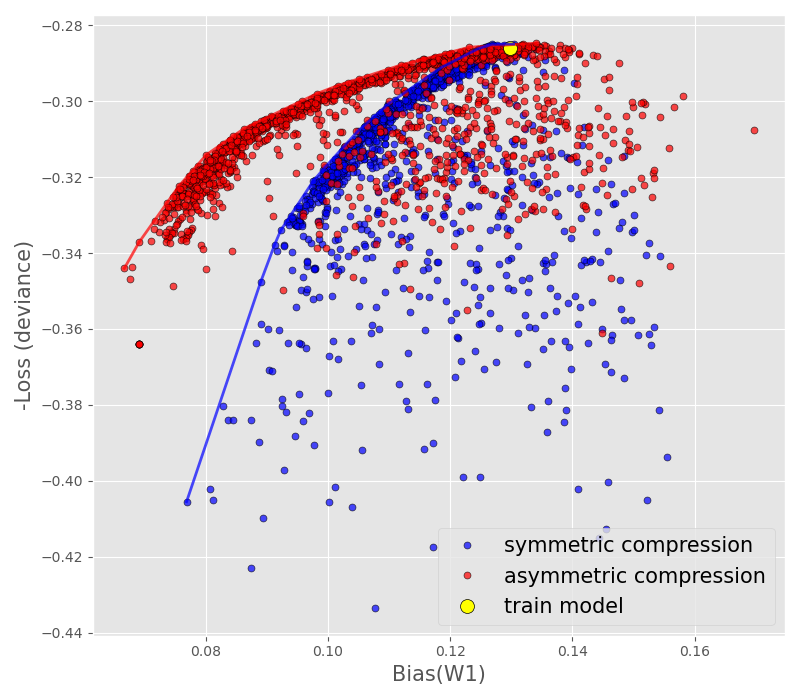}
    \caption{\footnotesize Bayesian search frontiers (test)}\label{fig::bayes_sym_vs_asym_m2}
  \end{subfigure}
 \caption{ Symmetric vs asymmetric for Model \eqref{realisticmodasym}.} \label{fig::bayes_info_m2}
\end{figure}

\subsubsection{Asymmetric compression}\label{sec::asymcomp}

\paragraph{\bf} Predictors whose subpopulation distributions have different variances will, in general, exhibit mixed bias explanations, where the positive and negative components are substantial. Rescaling such predictors will cause both of the components to either be compressed or expanded simultaneously, which is not always desirable. This motivates us to introduce asymmetric transformations that allow for rebalancing each of the positive and negative components independent of one another. In particular, consider the following transformation:
\begin{equation}\label{AsymComp}
T_{asym}(t; a_1,a_2, \sigma, t^*) =  \frac{(t-t^*)_{-}}{a_-} + \frac{(t-t^*)_{+}}{a_+} + t^*, \quad a_-,a_+>0
\end{equation}
which compresses differently around the focal point.

Consider the following data generating model:
\begin{equation}\label{realisticmodasym}\tag{M2}
\begin{aligned}
&\mu=5, \quad a=\tfrac{1}{10}(2.5,1.0,4.0,-0.25,0.75)\\
&X_1 \sim N(\mu-a_1 (1-G), 0.5+G \cdot 0.75), \quad X_2 \sim N(\mu-a_2 (1-G), 1 ) \\
&X_3 \sim N(\mu-a_3 (1-G), 1 ),     \quad X_4 \sim N(\mu-a_4 (1-G), 1-0.75 G ) \\
&X_5 \sim N(\mu-a_5 (1-G), 1) \\
&Y \sim Bernoulli(f(X)), \quad f(X)=\P(Y=1|X)={logistic(2 \big(\textstyle{\sum_{i}} X_i-24.5)\big)}.
\end{aligned}
\end{equation}
Note that in the model \eqref{realisticmodasym} there are two predictors $X_1, X_4$ with mixed bias explanations, while the rest have negative bias explanations equal to zero; Figure \ref{fig::bias_expl_m2}. Therefore, using symmetric transformations might not allow to fully take advantage of bias offsetting, since applying the transformation to $X_1, X_4$ would either simultaneously increase their positive and negative bias explanations, or decrease them. 

To this end, we compare our mitigation approach using symmetric and asymmetric transformations. First, we select the top two bias-impactful predictors in the list $N_+$ and the top two in the list $N_-$, which yields the list $M=\{1,3,4\} = \{ i_1,i_2,i_3 \}$. We then train a regularized GBM model with 10,000 samples and apply Algorithm \ref{BHPSalgo} using the transformations in \eqref{GlobComp} and in \eqref{AsymComp}. We consider the parameters $n_{prior}=400$, $n_{bo}=50$, $\{\omega_j = 2\cdot\frac{j}{20}\}_{j=0}^{20}$, $\gamma = (\alpha, x_M^*)$ where we set $x_M^* = \E[X_M]$ to be fixed and let the parameter $\alpha=\{ a_{i_k, \pm} \}_{k=1}^3$ vary with the bounds $a_{i_k,\pm} \in [0.5, 2]$. Figure \ref{fig::bayes_info_m2} illustrates that using asymmetric transformations allows to take full advantage of bias offsetting compared to symmetric ones because it can separately affect the positive and negative bias explanations of each predictor.

There are numerous transformations one can define that are useful in various scenarios. For instance, it may happen that subpopulation distributions of predictors differ only in a local region. In this case, transformations that reshape predictor distributions locally would be appropriate. Such an example is introduced in Appendix \ref{app::loccomp} and tested on the data generating model \eqref{realisticmodloc}.

\begin{figure}[t]
\centering
  \begin{subfigure}[t]{0.4\textwidth}
    \centering
    \includegraphics[width=\textwidth]{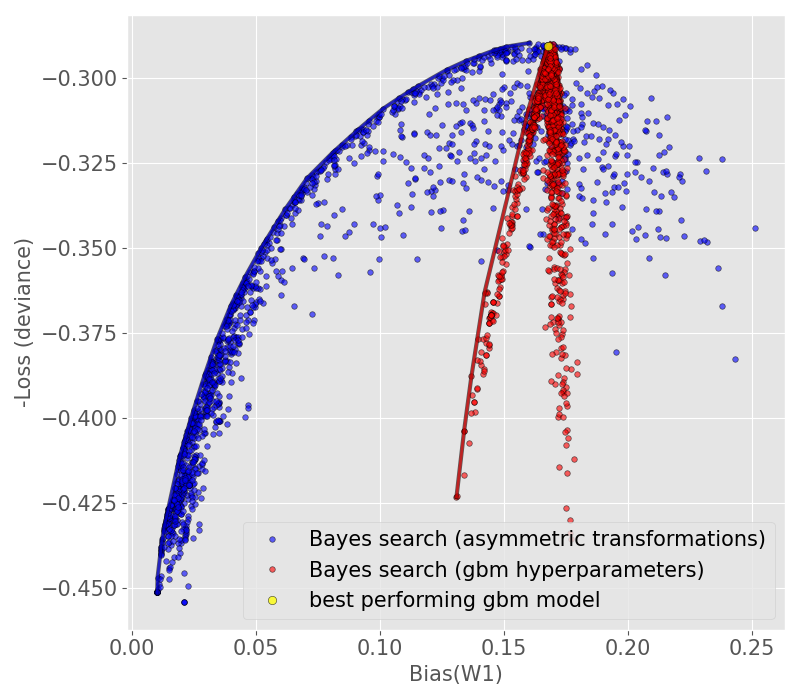}
    \caption{\footnotesize Model \eqref{realisticmod} }\label{fig::bayes_gbm_vs_asym_m1}
  \end{subfigure}
  ~~
  \begin{subfigure}[t]{0.4\textwidth}
    \centering
    \includegraphics[width=\textwidth]{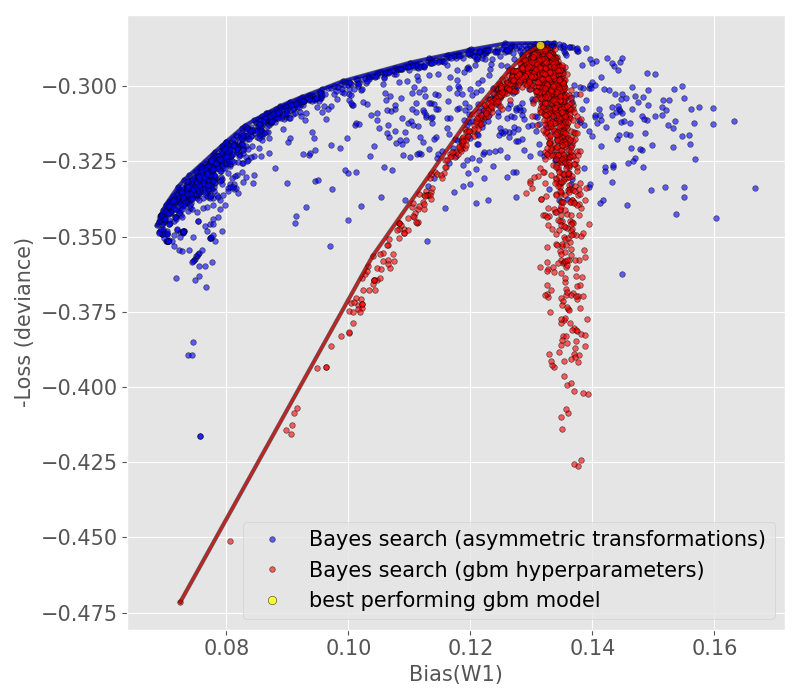}
    \caption{\footnotesize Model \eqref{realisticmodasym} }\label{fig::bayes_gbm_vs_asym_m2}
  \end{subfigure}
 \caption{ Comparison with Bayesian search over GBM-hyperparameter space.} \label{fig::gbm_asym_comp}
\end{figure}

\subsubsection{Comparison with Bayesian search over ML hyperparameter space}

One interesting comparison to carry out is between our mitigation procedure and that of \citet{Perrone2020} and \citet{Schmidt2021}. In \citet{Perrone2020} the bias methodology does not require  knowledge of the protected attribute $G$ either in training or prediction, and utilizes Bayesian optimization with fairness constraints on a wide range of models to learn ML hyperparameters that lead to fairer models. In \citet{Schmidt2021} the methodology randomly searches for ML hyperparameter configurations and builds the efficient frontier. Our mitigation approach shares some similarities with these methods. In particular, the methodologies in \citet{Perrone2020} are model-agnostic and do not require incorporating the fairness metric into the training process. The Bayesian search attempts to locate the efficient frontier by exploring the hyperparameter space and training a corresponding model. This leads to a two-step process, where one first optimizes with respect to performance and then makes a decision about a new hyperparameter configuration by taking into account a fairness constraint. In our methodology, there is only one step, which optimizes performance under a bias penalization, and the space of parameters is a low dimensional space of continuous parameters.

Another important difference is the fact that in our approach, post-processing the trained model completely avoids re-training. Furthermore, our approach may utilize a smaller dataset compared to the training set, which leads to a lower complexity for the same tolerance, as specified in \eqref{algocomplex}. Finally, knowing where the bias comes from and its magnitude, makes our method more precise in countering model bias, avoiding minimization procedures in high dimensional spaces.


To carry out the comparison, we use data generating models \eqref{realisticmod}-\eqref{realisticmodasym}. In step 1, using 10,000 samples of generated data for each aforementioned model, we apply Algorithm \ref{BHPSalgo}, with a list of penalization coefficients $\omega \in [0,4]$, by varying GBM hyperparameters with the following bounds: number of estimators in $[40,250]$, maximum number of leaves in $[4,20]$, maximum depth of the tree in $[2,20]$, and learning rate in $[0.05, 0.5]$. We next pick a trained model that corresponds to the best performing model in step 1 for $\omega = 0$ as the model that is subject for post-processing. This is done for a fair comparison. 
For each model \eqref{realisticmod}-\eqref{realisticmodasym} we pick the list $M$ of most bias-impactful predictors as in the examples of Section \eqref{sec::globcomp}-\eqref{sec::asymcomp}.
 We then apply Algorithm \ref{BHPSalgo} using the asymmetric transformation \eqref{AsymComp} with parameters as follows: $\gamma = (\alpha, x_M^*)$ where we set $x_M^*$ to be a component-wise median of $X_M$ and let the parameter $\alpha=\{a_{i_k,\pm}\}_{k \in M}$ vary with the bounds $a_{i_k,\pm} \in [0.5, 2]$. The generated efficient frontiers can be seen in Figure \ref{fig::gbm_asym_comp}.

\begin{figure}[t]
\centering
  \begin{subfigure}[t]{0.31\textwidth}
    \centering
    \includegraphics[width=\textwidth]{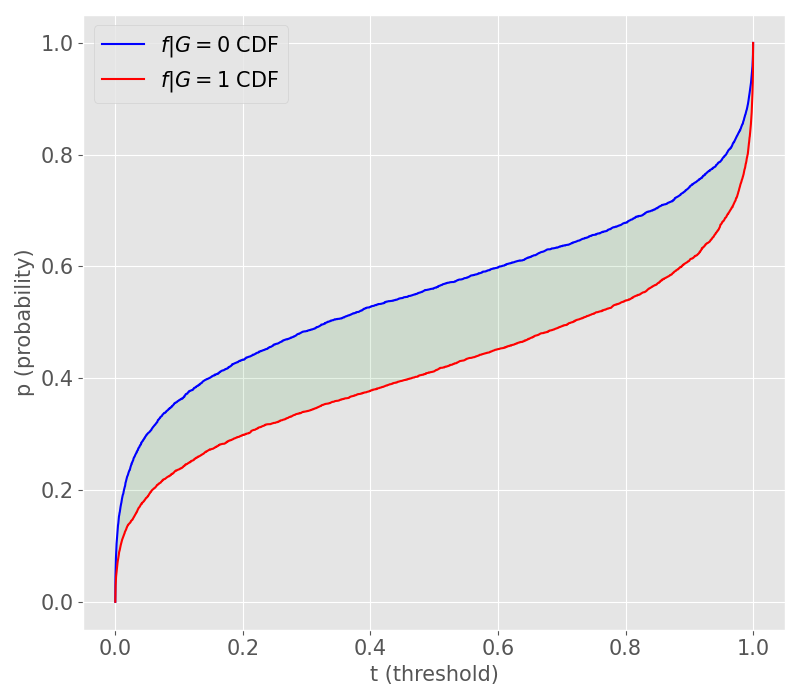}
    \caption{\footnotesize CDFs }\label{fig::subpop_distr_m4}
  \end{subfigure}
  \begin{subfigure}[t]{0.31\textwidth}
    \centering
    \includegraphics[width=\textwidth]{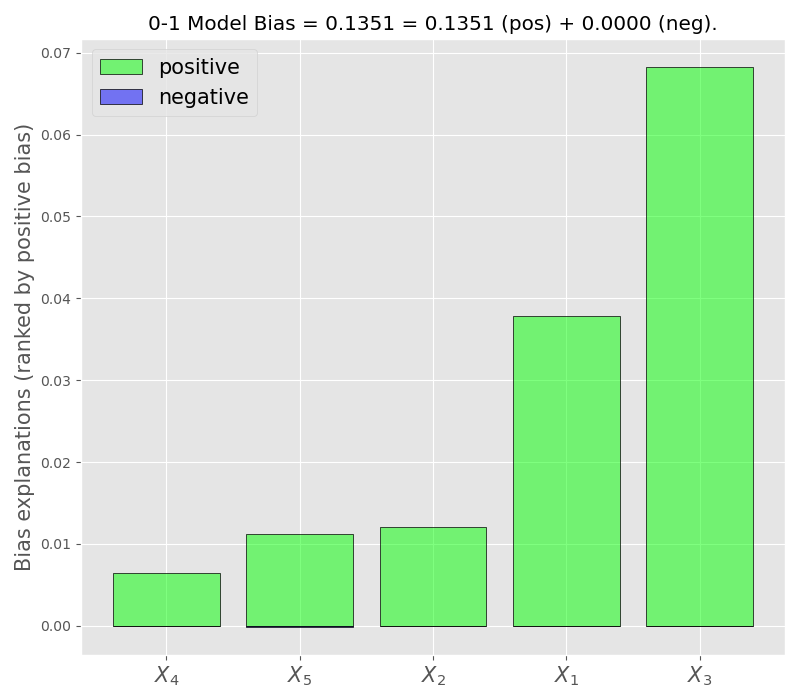}
    \caption{\footnotesize Bias explanations (train)}\label{fig::bias_expl_m4}
  \end{subfigure}
  \begin{subfigure}[t]{0.31\textwidth}
    \centering
    \includegraphics[width=\textwidth]{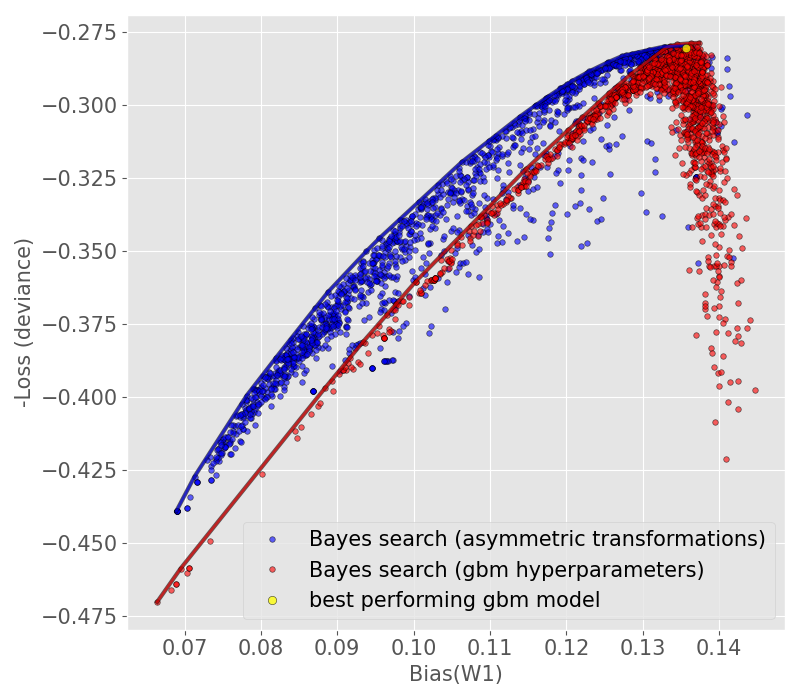}
    \caption{\footnotesize Bayesian frontier (test)}\label{fig::bayes_gbm_asym_m4}
  \end{subfigure}
 \caption{ Bayesian frontiers for Model \eqref{realisticmodpos}.} \label{fig::model_m4}
\end{figure}

A similar comparison is carried out for the data generating model
\begin{equation}\label{realisticmodpos}\tag{M3}
\begin{aligned}
&\mu=5, \quad a=\tfrac{1}{10}(2.5, 1.0,4,0.25,0.75)\\
&X_1 \sim N(\mu-a_1 (1-G), 1 ), \quad X_2 \sim N(\mu-a_2 (1-G), 1 ) \\
&X_3 \sim N(\mu-a_3 (1-G), 1 ),     \quad X_4 \sim N(\mu-a_4 (1-G), 1 ) \\
&X_5 \sim N(\mu-a_5 (1-G),1) \\
&Y \sim Bernoulli(f(X)), \quad f(X)=\P(Y=1|X)={logistic(2 \big(\textstyle{\sum_{i}} X_i-24.5)\big)}.
\end{aligned}
\end{equation}
where all predictors have positive bias explanations. This is a particular case which does not allow for offsetting; see Figure \ref{fig::model_m4}. From our own experience, we consider this an unrealistic case due to the fact that predictor subpopulations for many protected attributes are often imbalanced, which leads to mixed bias explanations.

It can be seen from Figure \ref{fig::bayes_gbm_asym_m4} that the frontier produced by our methodology is wider. It is our understanding that the difference in frontiers is caused by the fact that, in comparison, our methodology targets the most bias-impactful predictors and considers both performance and fairness in the minimization. The other method significantly impacts the performance while having low impact on bias, since fairness was not considered in the first step and the second step changes the ML parameters responsible for fairness, leading to a lower model ``resolution''.

\subsection{Model calibration}\label{sec::calib}

Unlike the trained model regressor, the post-processed model is no longer tied to data $(X,Y)$ because the construction involves algebraic transformations. Thus, while we can expect that $f\approx \E[Y|X]$, the post-processed model $\tilde{f}$ may no longer approximate the regressor; in the case of classification, the probabilities will be affected. Our goal is to remedy this issue by performing a calibration step after post-processing the original model $f$. To this end, we calibrate the model $\tilde{f}$ by constructing the final model in the form
\begin{equation}\label{eq::calibmap}
  \bar{f}(X) := C(\tilde{f}(X)\,;X, \tilde{f}, f)
\end{equation}
where $C$ is a calibrating map that arises in the process of  isotonic regression of either the trained model $f$ or the response variable $Y$ onto the post-processed model $\tilde{f}$. Calibration attempts to make $\bar{f}$ approximate $\E[Y|X]$.

\begin{definition}\label{def::calibration}
  Let $f_1, f_2$ be two models, $X$ a set of predictors, and $\mathcal{T}$ denote an admissible family of non-decreasing maps from $\RR$ to $\RR$. Given a loss function $L(z_1,z_2)$, isotonic regression of $f_2$ onto $f_1$ over $\mathcal{T}$ is a map $T^*:\RR\to\RR$  such that
  \[
    T^*(\cdot\,;X, f_1, f_2) = {\rm argmin} \Big\{\E[ L(f_2(X), T(f_1(X))];\, T \in \mathcal{T}\Big\}.
  \]
\end{definition}

\begin{figure}[t]
\centering
  \begin{subfigure}[t]{0.4\textwidth}
    \centering
    \includegraphics[width=\textwidth]{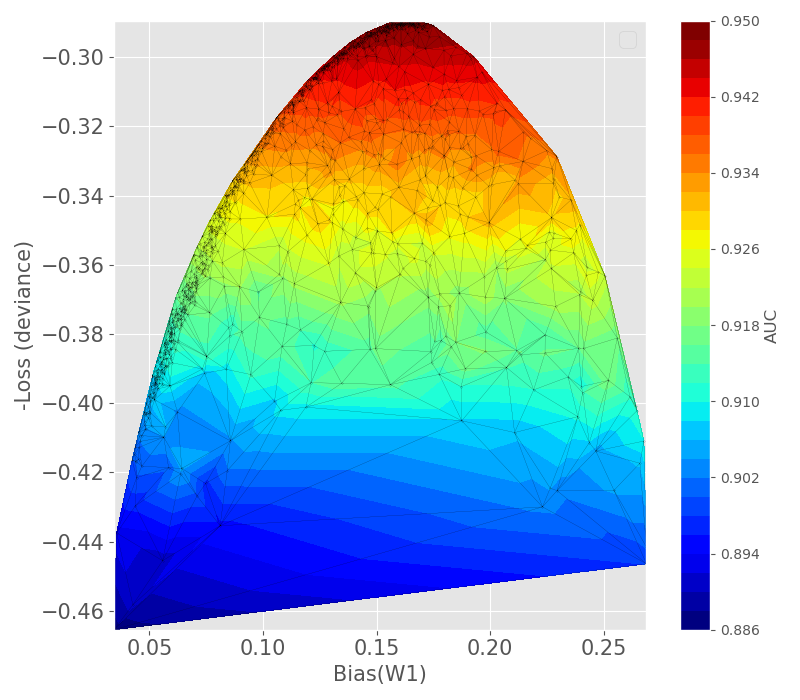}
    \caption{\footnotesize non-calibrated models (symmetric)  }\label{fig::noncalib_bhps_sym_m1}
  \end{subfigure}
  \begin{subfigure}[t]{0.4\textwidth}
    \centering
    \includegraphics[width=\textwidth]{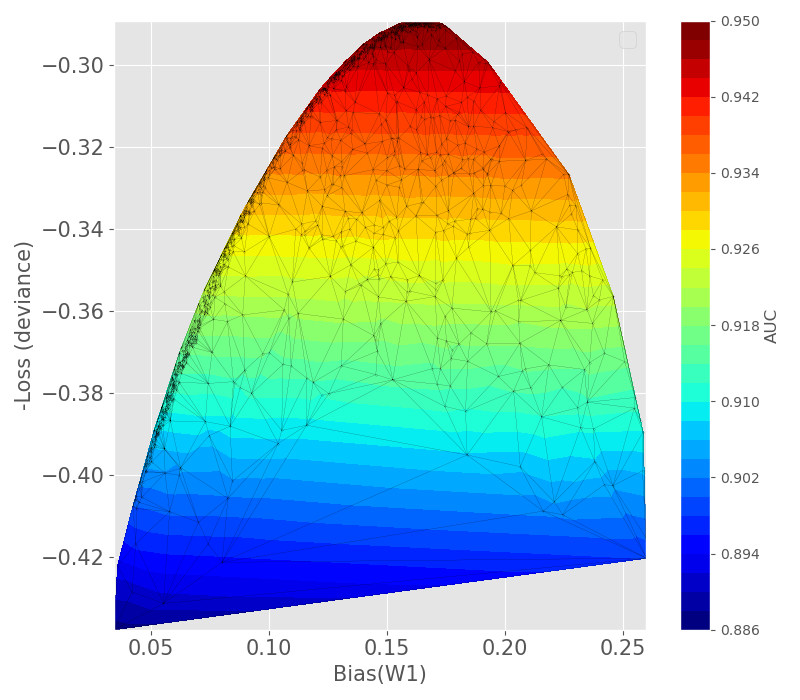}
    \caption{\footnotesize Calibrated models (symmetric) }\label{fig::calib_bhps_sym_m1}
  \end{subfigure}
  \begin{subfigure}[t]{0.4\textwidth}
    \centering
    \includegraphics[width=\textwidth]{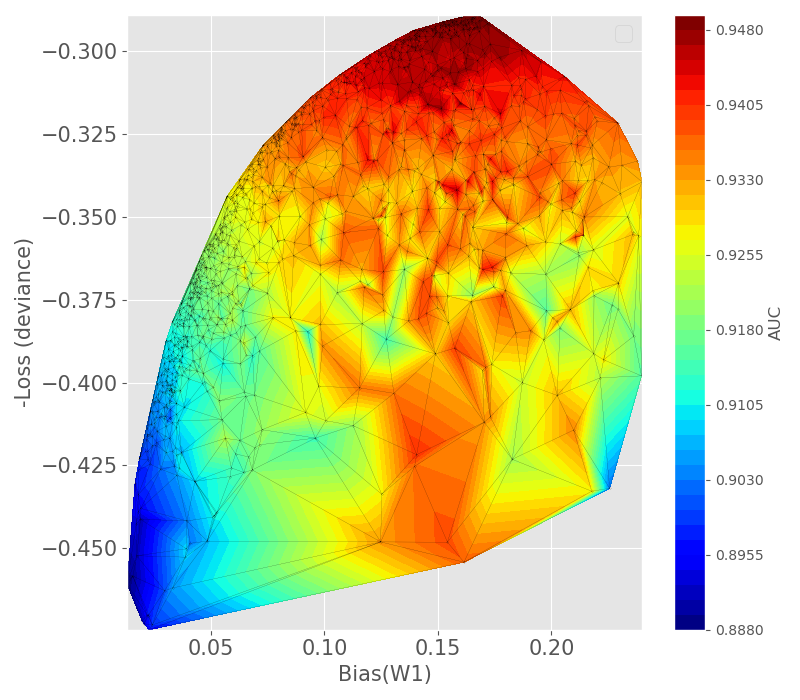}
    \caption{\footnotesize non-calibrated models (asymmetric)  }\label{fig::noncalib_bhps_asym_m1}
  \end{subfigure}
  \begin{subfigure}[t]{0.4\textwidth}
    \centering
    \includegraphics[width=\textwidth]{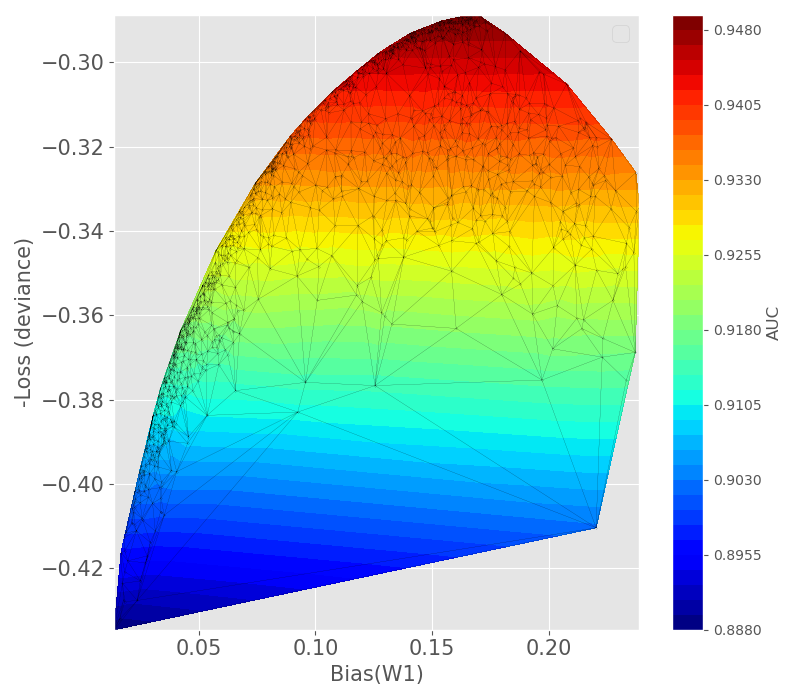}
    \caption{\footnotesize Calibrated models (asymmetric) }\label{fig::calib_bhps_asym_m1}
  \end{subfigure}
 \caption{ Calibration for model \eqref{realisticmod}.} \label{fig::calib_sym_m1}
\end{figure}

Given the above definition, the calibrating map $C$ in \eqref{eq::calibmap} is set to 
\[
  C(\cdot \,;X,\tilde{f}, f) = T^*(\cdot \,; X,\tilde{f}, f),
\]
given an appropriate admissible set $\mathcal{T}$ and loss function $L$. For calibration one can use the response variable $Y$ in place of $f$, but for classification regressors it is more appropriate to use the trained model rather than labels. Note that the procedure does not change classification properties for classification regressors, such as AUC, if $T^*$ is a continuous, strictly increasing map; it only affects the score distribution.

There are several ways to perform isotonic regression. Here we list some of them.
\begin{itemize}
  \item[(C1)] One popular method is non-smooth isotonic regression, that uses the loss $L(Z,\hat{Z}) = \E[(Z-\hat{Z})^2]$ and $\mathcal{T}=\{ \text{non-decreasing, piecewise linear functions} \}$. This method produces a regressor that has plateaus, resembling a non-decreasing step function.

  \item[(C2)] In many cases, it is desirable to have a smooth calibrating map. The work of \citet{Jiang2011} builds upon non-smooth isotonic regression which produces $T^*$ to be the map in the form
  \[
    T^*_{PCHIP} = T_{spline}\circ T_{nsi},
  \]
where $T_{nsi}$ is the non-smooth isotonic regressor and $T_{spline}$ is a map produced by fitting third-order splines to the points sampled from the graph of $T_{nsi}$. By subsampling, overfitting is countered and a smooth regressor is constructed.

  \item[(C3)] For classification regressors, one can obtain a smooth isotonic regressor by performing logistic regression of the response variable $Y\in \{0,1\}$ onto the post-processed score, that is, constructing $\widehat{P}(Y=1|\tilde{f}(X))$.

  \item[(C4)] A simple way to produce smooth isotonic regressors is to use generalized linear regression of the trained model onto the post-processed model. Specifically, given a link function $\sigma$, construct the functions
  \[
    g = \sigma^{-1}\circ f, \quad \tilde{g} = \sigma^{-1}\circ \tilde{f},
  \] 
  and then linearly regress $g(X)$ onto $\tilde{g}(X)$. For classification regressors the link function is typically chosen to be a logistic function. 
\end{itemize}

In all examples in this article we employ (C4) as our preferred method for calibration.

\paragraph{Application.} Consider the post-processed models constructed by performing Algorithm \ref{BHPSalgo} under symmetric transformations \eqref{GlobComp} to the trained model obtained by training on the dataset generated by the model \eqref{realisticmod}; see Figure \ref{fig::bayes_m1}. To understand the effect of calibration, we create a heat plot based on the bias-performance plot in Figure \ref{fig::bayes_m1}, where the color indicates the area under the ROC curve (AUC); see Figure \ref{fig::calib_bhps_sym_m1}. Since the AUC is invariant under continuous monotonic transformations, we may observe the effect of calibration by plotting a similar heat plot of non-calibrated models and noticing the change in level sets. Figures \ref{fig::noncalib_bhps_sym_m1}-\ref{fig::calib_bhps_sym_m1} depict the heat plots for calibrated and non-calibrated models for symmetric and Figures \ref{fig::noncalib_bhps_asym_m1}-\ref{fig::calib_bhps_asym_m1} for asymmetric transformations.








%% file: conclusion.tex
\section{Conclusion}
In this paper, we described a novel bias mitigation methodology that is based upon the construction of post-processed models with fairer regressor distributions for an extensive class of Wasserstein-based fairness metrics and selects the optimal post-processed model by constructing a Pareto efficient frontier over the family of post-processed models via Bayesian optimization. Our approach
performs optimization in low-dimensional spaces and avoids expensive model retraining. Furthermore, unlike many other mitigation techniques in the literature, our methodology does not explicitly utilize the protected attribute, which makes our method ideal for use in regulated environments. 
 

Our mitigation approach hinges upon the information extracted from the bias attributions, which can be decomposed into positive and negative bias attributions, and takes advantage of offsetting. This happens upon the inclusion of predictors that push the non-protected distribution in favorable and non-favorable directions. It may happen that all the bias attributions have negative bias explanations equal to zero, or vice versa. This means that each predictor contributes to pushing the non-protected class only in the favorable direction, or vice versa. In realistic datasets, this is an unlikely scenario. Our mitigation method can still apply in this case, but one should consider the impact to model performance. In the future, we aim to improve upon mitigating bias in these rare cases by introducing a new class of predictor transformations so that performance is not significantly impacted.

In its current setup, our experiments have shown that the Pareto efficient frontier constructed via Bayesian optimization and bias explanations performs significantly better than a simple random search. However, other possible optimization techniques may be able to push the frontier even further to obtain fairer models without sacrificing performance. Such techniques are based on gradient descent which are applicable when the trained model is smooth. As our work focused on tree-based models, which are not smooth, Bayesian optimization was the appropriate technique to apply.

%% file: acknowledgments.tex

\section*{Acknowledgments}
The authors would like to thank Steve Dickerson (SVP, CDO, Decision Management at Discover Financial Services (DFS)), Raghu Kulkarni (VP, Data Science at DFS)  and Melanie Wiwczaroski (Sr. Director, Enterprise Fair Banking at DFS) for formulation of the problem as well as helpful business and compliance insights. We also thank Melinda Milenkovich (VP \& Assistant General Counsel at DFS) and Kate Prochaska (Sr. Counsel \& Director, Regulatory Policy at DFS) for their helpful comments relevant to regulatory issues that arise in the financial industry.

%% file: appendix.tex
\begin{appendices}
\section*{Appendix}

\section{Bias explanations extensions}

\subsection{Individual bias explanations}\label{sec::ibes}

In this section, we provide a way of constructing bias explanations based on the aggregation across a family of explainers that serves as a generalization of the method discussed in \cite{Miroshnikov2020}.

\begin{definition}
Let $X \in \RR^n$ be predictors, $f$ a model, $\{E^{(\alpha)}_i(X)\}_{i=1}^n$  a family of predictor explainers parametrized by $\alpha \in \RR^m$, and $\{P_{i}(d\alpha)\}_{i=1}^n$ probability measures. The aggregated bias explanation of the predictor $X_i$ based on $(E_i^{(\alpha)},P_{i})$ is defined by
\[
\beta_i(f|X,G,E_{i}^{\alpha},P_{i}) := \E_{\alpha \sim P_i}[\Bias_{W_1,\A}^{(w)}(E_{i}^{\alpha}|X,G)]=\int \Bias_{W_1,\A}^{(w)}(E_{i}^{\alpha}|X,G) P_{\alpha,i}(d \alpha).
\]
\end{definition}

We note that the aggregated bias explanation of $X_i$ is an explanation where the bias in predictor explainers is averaged over the family with weights incorporated in the probability measure $P_{\alpha,i}$. The aggregation may be helpful when single explainers fail to produce consistent attributions.

For example, consider the following family of explainers 
\begin{equation}\label{iceexplfamily}
\alpha \in \RR^{n-1}, \quad E^{(\alpha)}_i(X)=f(X_i,x_{-\{i\}})|_{x_{-\{i\}}=\alpha}, \quad P_{i}(d \alpha)=P_{X_{-\{i\}}}(d\alpha),
\end{equation}
which is motivated by the individual conditional expectations as described in \citet{Goldstein et al}. 

Specifically, given a model $f$, a data set of predictors $\{x^{(j)}\}_{j=1}^{N}$, and the index $i \in \{1,2,\dots,n\}$, individual conditional expectations of predictor $X_i$ are constructed by investigating the maps $x_{i} \to f(x_i,x^{(j)}_{-\{i\}})$ for each data sample $x^{(j)}$.  While the word ``conditional'' pertains to the regressor $f$, the approach  itself is still partially marginal. In particular, the joint information on $X$ is ignored by separating the joint distribution of $X_{-\{i\}}$ from the distribution of $X_i$; as a consequence, the true interactions between $X_i$ and $X_{-\{i\}}$ might not be fully captured. Nevertheless, unlike marginal expectations (PDPs), this approach destroys far less information on the joint distribution. Thus, the ICE-based explainer family \eqref{iceexplfamily} motivates one to define the bias explanations as follows:

\begin{definition}\label{def::IBE}
Let $X$, $G$, $f$, $\mathcal{A}$ be as in Definition \ref{def::modbias} and let $E_i^{(\alpha)}$, $\alpha \in \RR^{n-1}$, be defined by \eqref{iceexplfamily}.

An individual bias explanation (IBE) of $X_i$ at $x_{\{-i\}}$ is defined as the bias in the explainer $E_i^{(x_{-\{i\}})}(X_i)$:
\[
\beta_i^{\IBE}(x_{-\{i\}},f|X,G)=\Bias_{W_1,\mathcal{A}}^{(w)}(E_i^{(x_{-\{i\}})}|X_i,G).
\]
The corresponding expected individual bias explanation of $X_i$ is defined as
\[
\bar{\beta}_i^{\IBE}(f|X,G)= \E[\beta_i^{\IBE}(X_{-\{i\}},f|X,G)]=\int \beta_i^{\IBE}(x_{-\{i\}},f|X,G) P_{X_{-\{i\}}}(dx_{-\{i\}}).
\]
The corresponding positive and negative  IBEs and expected IBEs are defined as follows:
\[
\begin{aligned}
\beta_i^{\IBE\pm}(x_{-\{i\}},f|X,G)]&=\Bias_{W_1,\mathcal{A}}^{(w)\pm}(E_i^{(x_{-\{i\}})}|X_i,G)
\\ \bar{\beta}_i^{\IBE\pm}(f|X,G)&=\E[\beta_i^{\IBE\pm}(X_{-\{i\}},f|X,G)]].
\end{aligned}
\]
\end{definition}

\paragraph{Example (bias explanations based on PDPs versus IBEs).} To illustrate the difference between PDPs and expected IBEs consider predictors $X\in \RR^2$ and a model $f(X)=f_1(X_1)f_2(X_2)$, and let $\A=\{\Omega\}$. Then the bias explanations for the explainer $E_i=\vpdp(\{i\};X,f)$, $i \in \{1,2\}$, have the following form:
\[
\beta^{\ME}_1=|\E[f_2(X_2)]|\Bias_{W_1}(f_1|X_1,G), \quad \beta^{\ME}_2=|\E[f_1(X_1)]|\Bias_{W_1}(f_2|X_2,G).
\]
Note that if $\E[f_i(X_i)]=0$ then $\beta^{\ME}_i=0$ for each $i \in \{1,2\}$, regardless of the bias level in $f(X)$.

For the expected IBEs, we have 
\begin{equation}\label{icebiasprod}
\bar{\beta}^{\IBE}_1 = \E[|f_2(X_2)|]\Bias_{W_1}(f_1|X_1,G), \quad \bar{\beta}^{\IBE}_2 = \E|f_1(X_1)| \Bias_{W_1}(f|X_2,G).
\end{equation}
Note that averaging in expected IBEs happens after computing the Wasserstein distance, which prevents the bias explanations from vanishing unless $f_i=0$ $P_{X_i}$-almost surely. 

Another example is motivated by \cite{Goldstein et al}. Consider the model
\[
f(X)=0.2 X_1 - 5X_2 + 10 X_2 \1_{\{X_3 \geq 0\}}
\]
where $\{X_i\}$ are identically distributed predictors satisfying $\E[X_i]=0$ and $\P(X_i \geq 0)=0.5$. In this case, $\vpdp(\{1\})=-0.2 X_1$, $\vpdp(\{2\})=\vpdp(\{3\})=0$ and hence the PDP-based bias explanations satisfy
\[
\beta^{\ME}_1=0.2 \cdot \Bias_{W_1}(X_1|G), \quad \beta^{\ME}_2=\beta^{\ME}_3=0.
\]
Once again, the averaging process vanishes the second and third predictor explanations in light of the interactions in the model, which leads to bias explanations incorrectly quantifying the bias contributions.

For expected IBEs, on the other hand, we have
\[
\bar{\beta}^{\IBE}_1=0.2\cdot \Bias_{W_1}(X_1|G), \quad \bar{\beta}^{\IBE}_2=5 \cdot \Bias_{W_1}(X_2|G), \quad \bar{\beta}^{\IBE}_3=10 \cdot \E[|X_2|]\Bias_{W_1}(\1_{\{X_3\geq 0\}}|G)
\]
which illustrates that the contributions of the second and third predictors are consistently captured. 

\begin{remark}\rm
We should point out that for additive models in the form $f(X)=\sum_{i=1}^n f_i(X_i)$ the expected IBEs and PDP-based bias explanations coincide, $\bar{\beta}^{\IBE}_i=\beta^{\ME}_i=\Bias_{W_1,\A}(f_i|X_i,G)$.
\end{remark}

\paragraph{Example (bias explanations based on Shapley values versus IBEs).} 
Shapley explanations often split the contributions between interacting predictors and hence the corresponding bias explanations will be split. The IBE-based explanations, consider an effect of isolated predictor without taking into account the interaction. This leads to a different ranking scheme in terms of the bias impact. Consider the following model,
\[
f(X)=X_1 X_2+X_3, \quad X\in \RR^3,
\]
where $\Bias_{W_1}(X_1|G)=0$ and $\Bias_{W_1}(X_i|G)=1$ for $i \in \{2,3\}$.

In this case, the expected IBEs  are given by
\[
\bar{\beta}_1^{\IBE}=0, \quad \bar{\beta}_2^{\IBE}=\E|X_1|, \quad \bar{\beta}_3^{\IBE}=1
\]
while for Shapley-based bias explanations (both conditional and marginal) we have
\[
\beta_1^{\SHAP}=\beta_1^{\SHAP}=\frac{1}{2} \Bias_{W_1}(X_1 X_2), \quad \beta_3^{\SHAP}=1.
\]

\section{Auxiliary lemmas}\label{app::auxlemmas}

\begin{lemma}\label{lmm::indepwasserstconn}
Let $Z \in \RR^n$ be a random vector and $H \in \{1,2,\dots,m\}$ a random variable. Let $D(\cdot,\cdot)$ be a metric on the space of probability measures $\mathscr{P}_q(\RR)$, with $q \geq 0$. Suppose that $\E[|Z|^q]$ is finite. Then $Z$ and $H$ are independent if and only if $D(Z|H=i,Z|H=j)=0$ for all $i,j$.
\end{lemma}
\begin{proof}
If $Z$ and $H$ are independent, then $P_{Z|H=i}=P_Z$, which implies  $D(Z|H=i,Z|H=j)=0$. Suppose now that $D(Z|H=i,Z|H=j)=0$ for all $i,j$. Then $P_{Z|H=i}=P_{Z|H=j}$ for all $i,j$ and hence for any Borel $A \subset \RR^n$ and $j \in \{1,2,...,m\}$, we obtain
\[
\P(Z \in A)=\sum_{i} \P(Z \in A|H=i) \P(H=i)=\P(Z \in A|H=j) \sum_{i} \P(H=i)=\P(Z \in A|H=j)
\]
which proves the statement.
\end{proof}

\begin{lemma}\label{lmm::comprbias}
  Let  $T(\cdot;a,t^*)$ be as in Definition \ref{def::comprprop}. Let $G \in \{0,1\}$ be a protected attribute such that $\PP(G=k)>0$, $k \in \{0,1\}$, and $Z$ be a random variable that satisfies $\E\big[|Z|^{max(1,q)}\big]<\infty$ for $q>0$ in Definition \ref{def::comprprop}(i). Then 
  \[\displaystyle \lim_{a\to +\infty} W_1(T(Z;a,t^*)|G=0, T(Z;a,t^*)|G=1)=0.
  \]
\end{lemma}

\begin{proof}
From the growth assumption in Definition \ref{def::comprprop}$(i)$, it follows that
\[
\E\big[|T(Z;a,t)-t_*|\big]\leq C \E\big[|Z-t^*|^{q}\big] a^{-\delta} \leq \tilde{C} \big(1 + \E\big[|Z|^{max(1,q)}\big] \big)a^{-\delta}.
\] 
Hence $T(Z;a,t_*)|G=k$, $k \in \{0,1\}$, converge to $t_*$ in $L^1(\PP)$ as $a\to+\infty$. This proves the lemma. 
\end{proof}

\noindent{\bf Proof of Lemma \ref{lmm::vanishbias}}
\begin{proof}
Without loss of generality, let us assume that $\A=\{\Omega\}$. First, using the triangle inequality, we obtain
\[
\begin{aligned}
& W_1(X|G=0,X|G=1) \\
& \leq W_1(X|G=0,X_{\eps}|G=0) +W_1(X_{\eps}|G=0,X_{\eps}|G=1)  + W_1(X_{\eps}|G=1,X|G=1)\\
& \leq \E[|X_{\eps}-X||G=0] +W_1(X_{\eps}|G=0,X_{\eps}|G=1) + \E[|X_{\eps}-X||G=1].
\end{aligned}
\]
By assumptions of the lemma, the right hand-side of the above inequality tends to zero as $\eps \to 0$ and hence $W_1(X|G=0,X|G=1)=0$. This implies that $X$ is independent of $G$. Hence $f(X)$ is independent of $G$, which implies $\Bias_{W_1}(f|X,G)=0$.

Similarly, by the triangle inequality and the fact that $f(X)$ is independent of $G$ we have
\begin{equation*}
\begin{aligned}
& W_1(f(X_{\eps})|G=0,f(X_{\eps})|G=1) \\
& \leq W_1(f(X_{\eps})|G=0,f(X)|G=0) + W_1(f(X)|G=1,f(X_{\eps})|G=1). 
\end{aligned}
\end{equation*}
By the continuity and boundedness of $f$, we have that $f(X_{\eps})$ converges to $f(X)$ in distribution. Hence $f(X_{\eps})|G=k$ converges to $f(X)|G=k$ in distribution for $k \in \{1,2\}$. Since $f$ is bounded, $\{P_{f(X_{\eps})}\}_{\eps}$ have common compact support and hence we must have
\[
W_1(f(X_{\eps})|G=k,f(X)|G=k) \to 0
\]
for $k\in\{0,1\}$ as $\eps \to 0$; see \citet[Corollary 6.13]{Villani2009}. Then from the last inequality it follows that $\Bias_{W_1}(f|X_{\eps},G) \to 0$ as $\eps \to 0$.

\end{proof}


\begin{figure}
\centering
  \begin{subfigure}[t]{0.3\textwidth}
    \centering
    \includegraphics[width=\textwidth]{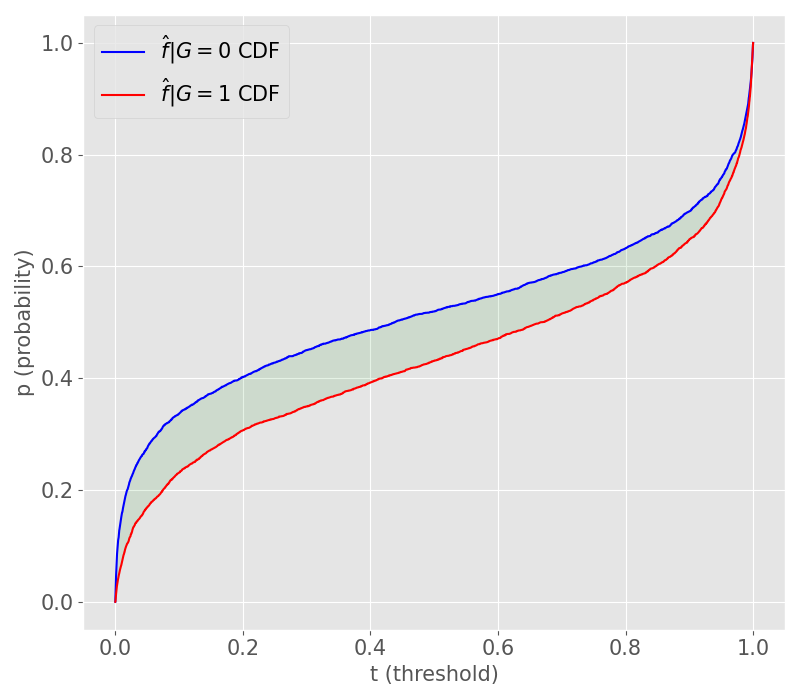}
    \caption{\footnotesize model CDFs }\label{fig::subpop_distr_m3}
  \end{subfigure}
~~
  \begin{subfigure}[t]{0.3\textwidth}
    \centering
    \includegraphics[width=\textwidth]{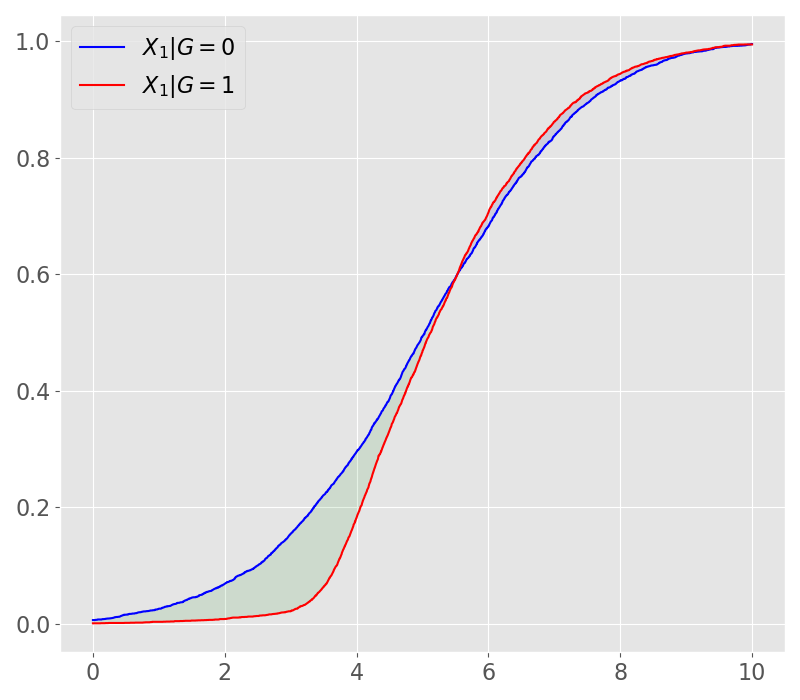}
    \caption{\footnotesize $X_1$ CDFs }\label{fig::subpop_cdf_x1_m3}
  \end{subfigure}
~~
  \begin{subfigure}[t]{0.3\textwidth}
    \centering
    \includegraphics[width=\textwidth]{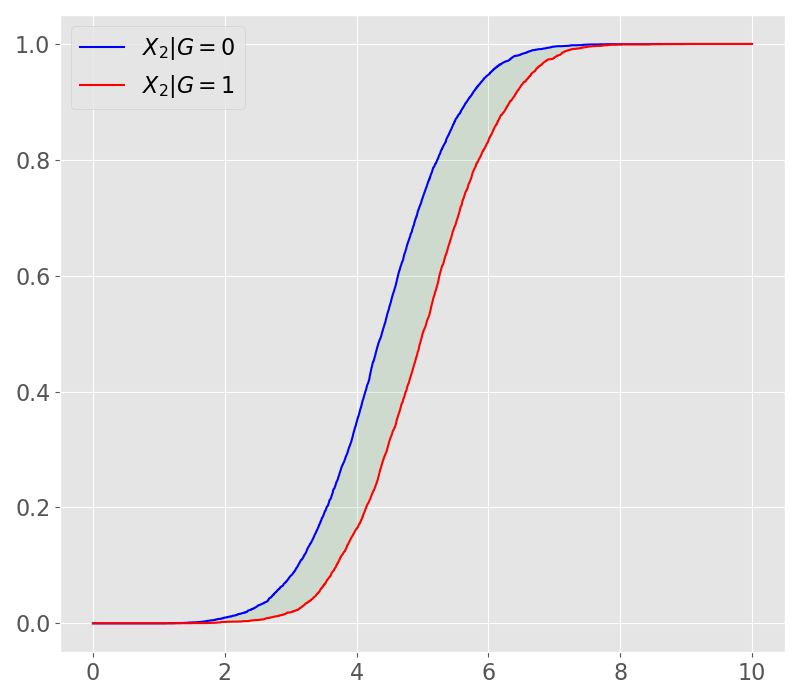}
    \caption{\footnotesize $X_2$ CDFs }\label{fig::subpop_cdf_x2_m3}
  \end{subfigure}
~~
  \begin{subfigure}[t]{0.3\textwidth}
    \centering
    \includegraphics[width=\textwidth]{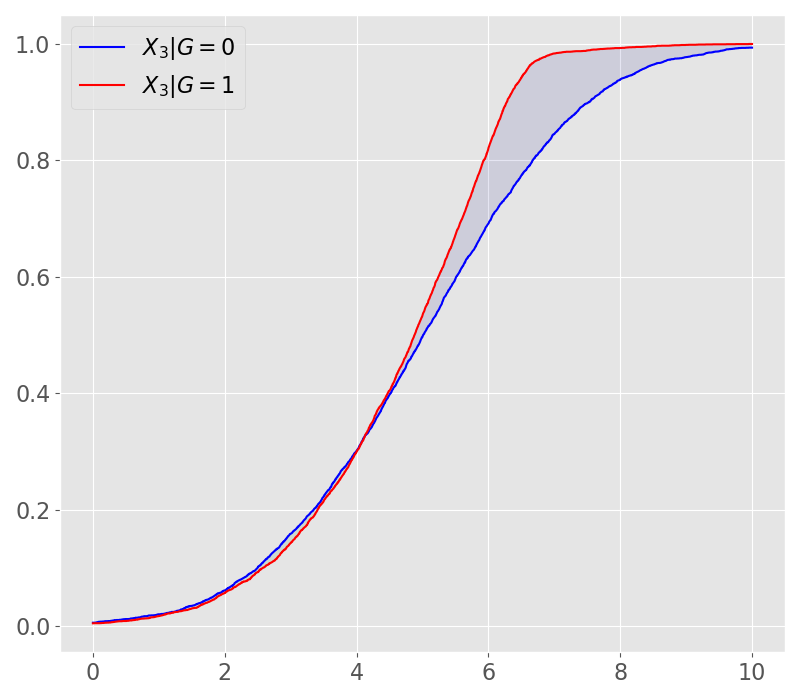}
    \caption{\footnotesize $X_3$ CDFs }\label{fig::subpop_cdf_x3_m3}
  \end{subfigure}
~~
  \begin{subfigure}[t]{0.3\textwidth}
    \centering
    \includegraphics[width=\textwidth]{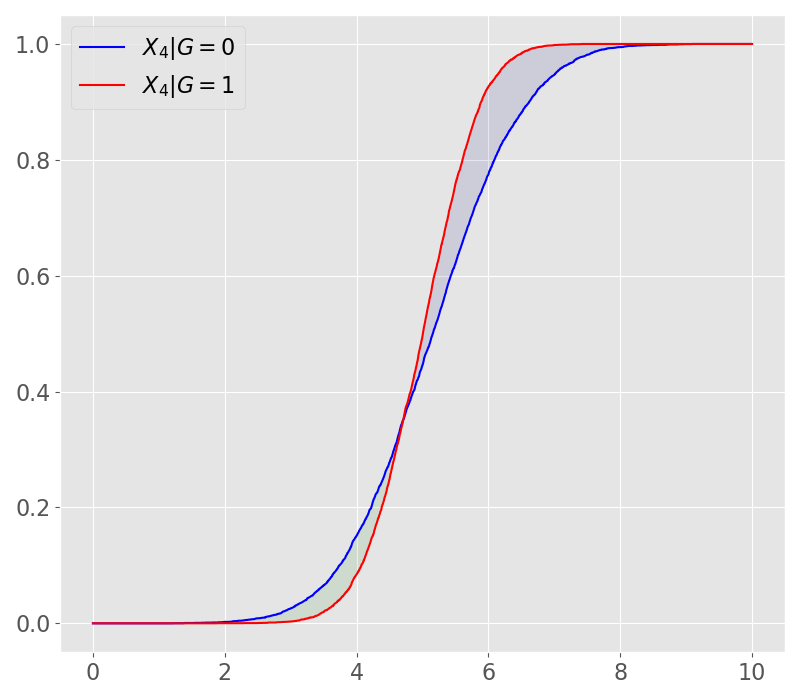}
    \caption{\footnotesize $X_4$ CDFs }\label{fig::subpop_cdf_x4_m3}
  \end{subfigure}
~~
  \begin{subfigure}[t]{0.3\textwidth}
    \centering
    \includegraphics[width=\textwidth]{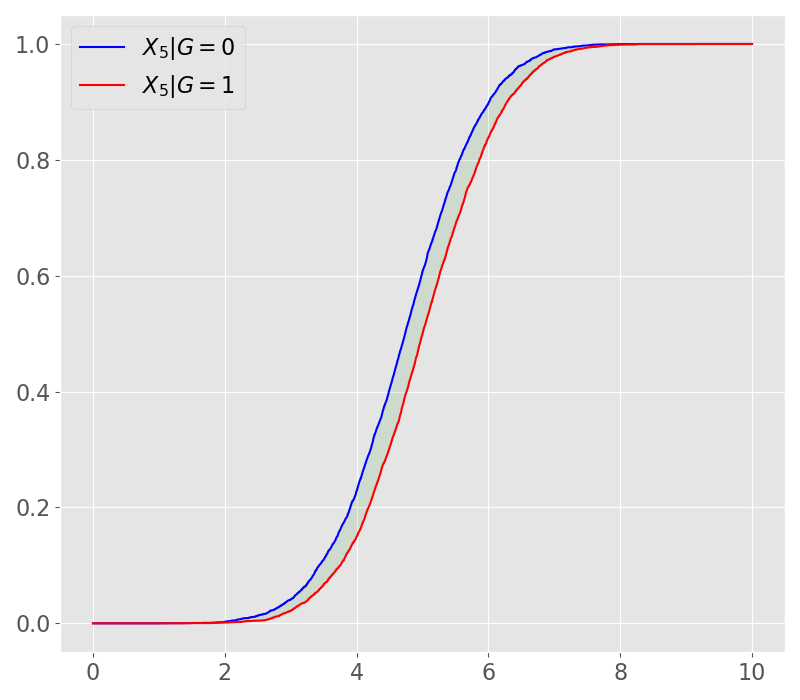}
    \caption{\footnotesize $X_5$ CDFs }\label{fig::subpop_cdf_x5_m3}
  \end{subfigure}
~~  
  \begin{subfigure}[t]{0.3\textwidth}
    \centering
    \includegraphics[width=\textwidth]{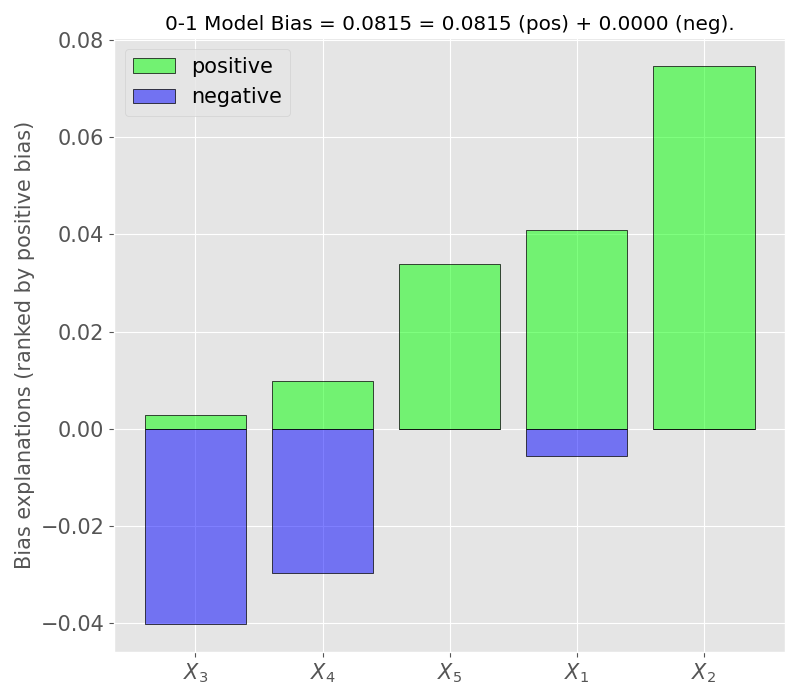}
    \caption{\footnotesize Bias explanations (train)}\label{fig::bias_expl_m3}
  \end{subfigure}
  \begin{subfigure}[t]{0.3\textwidth}
    \centering
    \includegraphics[width=\textwidth]{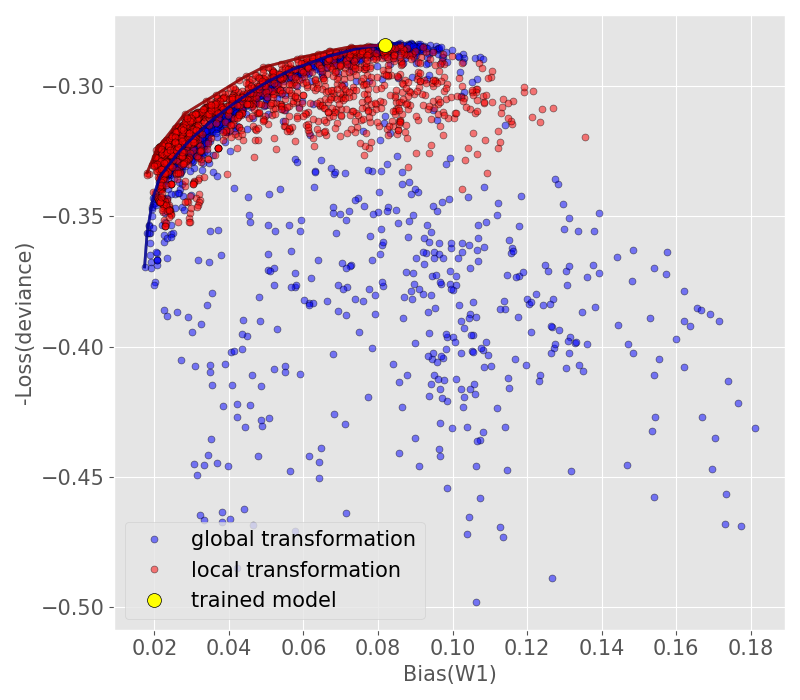}
    \caption{\footnotesize Bayesian search frontiers (test)}\label{fig::bayes_loc_vs_glob_m3}
  \end{subfigure}
 \caption{ Local vs global for Model \eqref{realisticmodloc}.}\label{fig::model_m3}
\end{figure}

\section{Local compression} \label{app::loccomp}
In some cases, it may be that subpopulation distributions of predictors differ only in a local region. This makes reshaping the predictor values globally less efficient, since not all values require rescaling. This motivates us to introduce local compressions of the following form
\begin{equation}\label{LocalComp}
T_{loc}(t; a,\sigma, t^*) =  t - (t-t^*) (1-a^{-1}) \exp\Big\{ -\frac{(t-t^*)^2}{2\sigma^2} \Big\},
\end{equation}
where
\[
|t-t^*| \lesssim \sigma \Rightarrow T(t; a, \sigma ,t^*) \approx \frac{t-t*}{a}+t^*, \quad |t-t^*| \gg  \sigma \Rightarrow T(t; a, \sigma ,t^*) \approx t.
\]

The following data generating model is introduced to further illustrate the application of local transformations and compare them to global rebalancing. 
\begin{equation}\label{realisticmodloc}\tag{M4}
\begin{aligned}
&\mu=5, s = 1.6\\
&X_1 = Z_1 + (Z_2-Z_1)(\1_{\{Z_0>\mu+s\}}\1_{\{Z_0<\mu-s\}}), \quad X_2 \sim N(\mu-0.6 (1-G), 1 ),\\
&X_3 = Z_1 + (Z_4-Z_1)(\1_{\{Z_3>\mu+s\}}\1_{\{Z_3<\mu-s\}}),\quad X_4 \sim N(\mu+0.15 (1-G), 1.25-0.75 G ) \\
&X_5 \sim N(\mu-0.45 (1-G), 1) \\
&Y \sim Bernoulli(f(X)), \quad f(X)=\P(Y=1|X)={logistic(1.5 \big(\textstyle{\sum_{i}} X_i-24)\big)},\\
\end{aligned}
\end{equation}
where
\[
\begin{aligned}
& Z_0 \sim N(\mu, 1.25), \quad Z_1 \sim N(\mu, 2), \quad Z_3 \sim N(\mu, 1)\\
&Z_2 \sim f(x) = \frac{2}{2.4} \phi(\tfrac{1}{2.4}(x-\mu+1.5)) \Phi(\tfrac{8}{2.4}(x-\mu+1.5)),\\
&Z_4 \sim f(x) = \frac{2}{2.4} \phi(\tfrac{1}{2.4}(x-\mu-1.5)) \Phi(-\tfrac{1}{2.4}(x-\mu-1.5)).
\end{aligned}
\]
and the functions $\phi, \Phi$ are the PDF and CDF of a standard normal, respectively.

In this example, we compare the use of global and local transformations in the form \eqref{GlobComp} and \eqref{LocalComp}, respectively. First, we select the top two bias-impactful predictors in the list $N_+$ and the top two in the list $N_-$, which yields the list $M=\{1,2,3,4\} = \{ i_1,i_2,i_3, i_4 \}$; see Figure \ref{fig::bias_expl_m3}. We then train a regularized GBM model with 10,000 samples and apply Algorithm~\ref{BHPSalgo}. For the experiment with the global transformation \eqref{GlobComp} we use the same parameters as in the example of Section \ref{sec::globcomp}. For the experiment with the local transformation \eqref{LocalComp} we pick the parameters as follows : $n_{prior}=400$, $n_{ob}=50$, $\{\omega_j = 2\cdot\frac{j}{20}\}_{j=0}^{20}$, $\gamma = (\{a_{i_k},\sigma_{i_k}\}_{k=1}^4, x_M^*)$, where we have the bounds $a_{i_k} \in [0.5, 4.0]$ and $\sigma_{i_k}\in[1,2]$, and where $x^*_M$ is a fixed point. In both experiments $x^*_{i}$ is the median of $X_i$ for $i\in\{2,4\}$ and, for $i \in \{1,3\}$, it is the point where the value of KS distance between the two subpopulation distributions of $X_i$ is achieved.

Figure \ref{fig::model_m3} depicts the CDFs of each predictor in the model \eqref{realisticmodloc}. Observe that the predictors $X_1, X_3$ have subpopulations that differ locally, while the rest of the predictors exhibit differences in their subpopulations globally. Specifically, the difference in $X_2$ is obtained by a shift and for $X_4$ by changing variance. Figure \ref{fig::bias_expl_m3} depicts the bias explanations for the predictors that reflect the bias contributions to the model. The efficient frontiers obtained by transforming the predictors locally and globally are shown in Figure \ref{fig::bayes_loc_vs_glob_m3}, which shows that the frontier corresponding to the local transformation is shifted upwards. This signifies that local compression results in more appropriate post-processed models compared to global compression.

\end{appendices}